\documentclass[preprint,11pt]{article}

\usepackage{comment,url,algorithm,algorithmic,graphicx,subfigure}
\usepackage{amssymb,amsfonts,amsmath,amsthm,amscd,dsfont,mathrsfs,mathtools}
\usepackage{float,psfrag,epsfig,color,url,hyperref}
\usepackage{epstopdf,bbm,mathtools,enumitem}

\makeatletter
\@addtoreset{equation}{section}
\makeatother

\newtheorem{theorem}{Theorem}[section]

\newtheorem{lemma}[theorem]{Lemma}

\newtheorem{corollary}[theorem]{Corollary}
\newtheorem{definition}{Definition}
\newtheorem{remark}{Remark}

\definecolor{Red}{rgb}{1,0,0}
\definecolor{Blue}{rgb}{0,0,1}
\definecolor{Olive}{rgb}{0.41,0.55,0.13}
\definecolor{Green}{rgb}{0,1,0}
\definecolor{MGreen}{rgb}{0,0.8,0}
\definecolor{DGreen}{rgb}{0,0.55,0}
\definecolor{Yellow}{rgb}{1,1,0}
\definecolor{Cyan}{rgb}{0,1,1}
\definecolor{Magenta}{rgb}{1,0,1}
\definecolor{Orange}{rgb}{1,.5,0}
\definecolor{Violet}{rgb}{.5,0,.5}
\definecolor{Purple}{rgb}{.75,0,.25}
\definecolor{Brown}{rgb}{.75,.5,.25}
\definecolor{Grey}{rgb}{.5,.5,.5}
\definecolor{Pink}{rgb}{1,0,1}
\definecolor{DBrown}{rgb}{.5,.34,.16}
\definecolor{Black}{rgb}{0,0,0}

\def\l{\ell}
\def\xx{\alpha}
\def\ind{\mathbbm{1}}
\def\se{{\psi_1}}
\def\sg{{\psi_2}}
\def\N{\mathcal{N}_{[]}}
\def\F{\mathcal{F}}
\def\EE{\mathfrak{E}}

\def\tol{{\epsilon}}
\def\dphi{{\phi^{(1)}}}
\def\ddphi{{\phi^{(2)}}}
\def\ddddphi{{\phi^{(4)}}}
\def\X{\bold{X}}
\def\Q{\bold{Q}}

\def\I{\bold{I}}
\def\hmu{\hat{\mu}}
\def\hth{\hat{\beta}}
\def\th{{\beta}}
\def\tu{{\tilde{u}}}
\def\tl{{\tilde{l}}}
\def\e{{\epsilon}}
\def\hsig{{\hat{\sigma}}}
\def\T{{\mathcal T}}
\def\Hh{{\mathcal H}}
\def\reals{{\mathbb R}}

\def\<{\langle}
\def\>{\rangle}
\def\E{{\mathbb E}}
\def\P{{\mathbb P}}
\def\proj{{\mathcal P}}
\def\grad{{\nabla}}
\def\gradS{\boldsymbol{\nabla}^2}

\def\normal{{\sf N}}
\def\Sig{{\boldsymbol \Sigma}}
\def\O{{\mathcal O}}
\def\QQ{{\mathcal Q}}
\def\D{\mathfrak{D}}

\def\argmin{{\rm argmin}}
\def\trace{{\rm Tr}}
\def\SigH{\widehat{\boldsymbol {\Sigma}}}
\def\ALG{NewSt }
\def\ALGv{\text{Newton-Stein method} }

\newcommand{\eq}[1]{\begin{alignat}{3}#1\end{alignat}}
\newcommand{\eqn}[1]{\begin{alignat*}{3}#1\end{alignat*}}
\newcommand{\commentout}[1]{}

\begin{document}

\title{Newton-Stein Method: \\
{\fontsize{17}{6}\selectfont An optimization method for GLMs 
via Stein's Lemma}}

\author{
Murat A. Erdogdu\thanks{Department of Statistics,
Stanford University}
}

\date{}

\maketitle

\begin{abstract}
We consider the problem of efficiently computing the maximum likelihood estimator in 
\emph{Generalized Linear Models} (GLMs)
when the number of observations is much larger than the number of coefficients ($n\!\gg\! p\! \gg\! 1$). 
In this regime, 
optimization algorithms can 
immensely benefit from
approximate second order information.
We propose an alternative way of constructing the curvature information by formulating
it as an estimation problem and 
applying a \emph{Stein-type lemma}, which allows further improvements through sub-sampling and
eigenvalue thresholding.
Our algorithm 
enjoys fast convergence rates, resembling that of second order methods, 
with modest per-iteration cost. 
We provide its convergence analysis for the general case 
where the rows of the design matrix are samples from a sub-gaussian distribution.
We show that the convergence has two phases, a
quadratic phase followed by a linear phase. 
Finally,
we empirically demonstrate that our algorithm 
achieves the highest performance
compared to various algorithms on several datasets.

\end{abstract}


\section{Introduction}

Generalized Linear Models (GLMs) play a crucial role in 
numerous statistical and machine learning problems.
GLMs formulate the natural parameter in 
exponential families as a linear model and provide a miscellaneous 
framework for statistical methodology and 
supervised learning tasks. Celebrated examples include linear, logistic, 
multinomial regressions and applications to graphical models
\cite{nelder1972generalized,mccullagh1989generalized,koller2009probabilistic}.

In this paper, we focus on how to solve the maximum likelihood problem
efficiently in the GLM setting when 
the number of observations $n$ is much larger than the dimension 
of the coefficient vector $p$, i.e., $n \gg p \gg 1$. 
GLM optimization task is typically expressed as a minimization problem 
where the objective function is the negative log-likelihood that is denoted by $\ell(\th)$ where $\th \in \reals^p$ is the coefficient vector. Many optimization algorithms are available for such minimization problems \cite{bishop1995neural,Boyd:2004,nesterov2004introductory}. However, only a few uses the special structure of GLMs. In this paper, we consider updates that are specifically designed for GLMs, which are of the from
\eq{\label{eq::update}
\th \leftarrow \th -\gamma \Q \grad_\th \ell(\th)\, ,
}
where $\gamma$ is the step size and $\Q$ is a scaling matrix 
which provides curvature information.

For the updates of the form Eq.~(\ref{eq::update}), the performance of the algorithm is mainly determined 
by the scaling matrix $\Q$. Classical \emph{Newton's Method} (NM) and \emph{Natural Gradient} (NG) descent can be recovered by simply taking $\Q$ to be the inverse Hessian and the inverse Fisher's information at the current iterate, respectively \cite{amari1998natural,nesterov2004introductory}. Second order methods may achieve
 quadratic convergence rate, 
 yet they suffer from excessive cost of computing the scaling matrix at every iteration. 
On the other hand, if we take $\Q$ to be the identity matrix, we recover the simple \emph{Gradient Descent} (GD) method which has a linear convergence rate. Although GD's convergence rate is slow compared to 
that of second order methods such as NM, modest per-iteration cost makes it practical for large-scale optimization. 

The trade-off between the convergence rate and per-iteration cost has been extensively studied \cite{bishop1995neural,Boyd:2004,nesterov2004introductory}. In $n \gg p$ regime,
the main objective is to construct a scaling matrix $\Q$ that is computational feasible and provides 
sufficient curvature information. 
For this purpose, several Quasi-Newton methods have been proposed \cite{bishop1995neural,nesterov2004introductory}.
Updates given by Quasi-Newton methods satisfy an equation which is often called \emph{Quasi-Newton relation}.
A well-known member of this class of algorithms is the 
\emph{Broyden-Fletcher-Goldfarb-Shanno} (BFGS) algorithm 
\cite{broyden1970convergence,fletcher1970new,goldfarb1970family,shanno1970conditioning}.

In this paper, we propose an algorithm which utilizes the special structure of GLMs by relying on a Stein-type lemma \cite{stein1981estimation}. It attains fast convergence rates with low per-iteration cost. We call our algorithm \emph{Newton-Stein Method} which we abbreviate as \emph{\ALG}\!. Our contributions can be summarized as follows:
\begin{itemize}
\item We recast the problem of constructing a scaling matrix as an estimation problem and
apply a Stein-type lemma along with sub-sampling
to form a computationally feasible $\Q$.
\item Our formulation allows further improvements through sub-sampling techniques and eigenvalue thresholding.
\item Newton method's $\O(np^2+p^3)$ per-iteration cost is replaced by $\O(np+p^2)$ per-iteration cost and a one-time $\O(n|S|^2)$ cost, where $|S|$ is the sub-sample size.
\item Assuming that the rows of the design matrix are i.i.d. and have bounded support (or sub-gaussian), and denoting the iterates of Newton-Stein Method by $\{\hth^t\}_t$, we prove a bound of the form
\eq{\label{eq::boundMainResult}
\|\hth^{t+1} - \th_*\|_2
\leq \tau_1 \|\hth^t - \th_*\|_2+ \tau_2 \|\hth^t - \th_*\|_2^2,
}
where $\th_*$ is the minimizer and $\tau_1$, $\tau_2$ are the convergence coefficients. The above bound implies that the convergence starts with a quadratic phase and transitions into linear as the iterate gets closer to $\th_*$. 
\item We demonstrate the performance of \ALG on four datasets by comparing it to
commonly used algorithms.
\end{itemize}

The rest of the paper is organized as follows:
Section \ref{sec::relatedWork} surveys the related work and Section \ref{sec::notation} introduces the notations used throughout the paper. Section \ref{sec::glm} briefly discusses the GLM framework and its relevant properties.
In Section \ref{sec::algorithm}, we introduce \ALGv\!\!, develop its intuition, and discuss the computational aspects.
Section \ref{sec::theory} covers the theoretical results and in 
Section \ref{sec::parameters} we discuss how to choose the algorithm parameters. 
Section \ref{sec::experiments} provides the empirical results where we compare the proposed algorithm with several other methods on four datasets. 
Finally, in Section \ref{sec::discussion}, we conclude with a brief discussion along with a few open questions.

\subsection{Related work} \label{sec::relatedWork}

There are numerous optimization techniques that can be used to find the maximum likelihood estimator in GLMs. 
For moderate values of $n$ and $p$, classical second order methods such as NM are commonly used.
In large-scale problems, data dimensionality is the main factor while choosing the right optimization method. 
Large-scale optimization has been extensively studied through online and batch methods. Online methods use a gradient (or sub-gradient) of a single, randomly selected observation to update the current iterate \cite{robbins1951stochastic}. Their per-iteration cost is independent of $n$, but the convergence rate might be extremely slow. There are several extensions of the classical stochastic descent algorithms (SGD), providing significant improvement and/or stability \cite{bottou2010,Duchi11,schmidt2013minimizing}.

On the other hand, batch algorithms enjoy faster convergence rates, though their per-iteration cost may be prohibitive. In particular, second order methods enjoy quadratic convergence, but constructing the Hessian matrix generally requires excessive amount of computation. Many algorithms aim at formimg an approximate, cost-efficient scaling matrix. In particular, this idea lies at the core of Quasi-Newton methods \cite{bishop1995neural,nesterov2004introductory}. 

Another approach to construct an approximate Hessian makes use of sub-sampling techniques \cite{martens2010deep,byrd2011use,VinyalsAISTATS12,erdogdu2015convergence-long}. Many contemporary learning methods rely on sub-sampling as it is simple and it provides significant boost over the first order methods. Further improvements through conjugate gradient methods and Krylov sub-spaces are available. Sub-sampling can also be used to obtain an approximate solution, with certain large deviation guarantees
\cite{Dhillon-etal2013}.

There are many composite variants of the aforementioned methods, that mostly combine two or more techniques. Well-known composite algorithms are the combinations of sub-sampling and Quasi-Newton \cite{schraudolph2007stochastic,byrd2014stochastic}, 
SGD and GD \cite{friedlander2012hybrid}, 
NG and NM \cite{roux2010fast}, NG and low-rank approximation \cite{le2007topmoumoute}, sub-sampling and eigenvalue thresholding \cite{erdogdu2015convergence-long}. 

Lastly, algorithms that specialize on certain types of GLMs include
coordinate descent methods for the penalized GLMs \cite{friedman2010regularization} and trust region Newton methods \cite{lin2008trust}.

\subsection{Notation}\label{sec::notation}

Let $[n]=\{ 1,2,...,n\}$ and denote by $|S|$, the size of a set $S$. 
The gradient and the Hessian of $f$ with respect to $\th$ are denoted by $\grad_\th f$ and $\gradS_\th f$, respectively. 
The $j$-th derivative of a function $f(w)$ is denoted by $f^{(j)}(w)$.
For a vector $x$ and a symmetric matrix $\mathbf{X}$, $\|x\|_2$ and $\|\mathbf{X}\|_2$ denote the $\ell_2$ and spectral norms of $x$ and $\mathbf{X}$, respectively. 
$\|x\|_\sg$ denotes the sub-gaussian norm, which will be defined later. 
$S^{p-1}$ denotes the $p$-dimensional sphere. 
$\proj_\mathcal{C}$ denotes the Euclidean projections onto the set $\mathcal{C}$, and $B_p(R)\subset \reals^p$ denotes the 
ball of radius $R$.
For a random variable $x$ and density $f$, $x \sim f$ means that the distribution of $x$ follows the density $f$. Multivariate Gaussian density with mean $\mu \in \reals^p$ and covariance $\Sig\in\reals^{p \times p}$ is denoted as $\normal_p(\mu,\Sig)$. 
For random variables $x,y$, $d(x,y)$ and $\D(x,y)$ denote probability metrics (to be explicitly defined later) measuring the distance between the distributions of $x$ and $y$.
$\N(\cdots)$ and $T_\e$ denote the bracketing number and $\e$-net.


\section{Generalized Linear Models}\label{sec::glm}

Distribution of a random variable $y\in\reals$ belongs to an exponential family with natural parameter $\eta \in \reals$ if its density is
\eqn{
f(y | \eta ) = e^{\eta y - \phi (\eta)}h(y), 
}
where $\phi$ is the \emph{cumulant generating function} and $h$ is the \emph{carrier density}.
Let $y_1,y_2,...,y_n$ be independent observations such that $\forall i \in [n],$ $y_i \sim f(y_i | \eta_i ).$
Denoting $\eta = (\eta_1,...,\eta_n)$,  the joint likelihood can be written as
\eq{
f(y_1,y_2,...,y_n | \eta ) = \text{exp}\left\{\sum_{i=1}^n\left [ y_i\eta_i - \phi (\eta_i)\right ]\right\} \prod_{i=1}^nh(y_i).
}
We consider the problem of learning the maximum likelihood estimator in the above exponential family framework, where the vector $\eta\in \reals^{n}$ is modeled through the linear relation,
\eqn{
 \mathbf{\eta} = \X \th, 
}
\!for some design matrix $\X \in \reals^{n \times p}$ with rows $x_i\in \reals^p$, and a coefficient vector $\th\in\reals^{p}$. This formulation is known as \emph{Generalized Linear Models} (GLMs). The cumulant generating function $\phi$ determines the class of GLMs, i.e., for the ordinary least squares (OLS) $\phi(z)=z^2$ and for the logistic regression (LR) $\phi(z)=\log(1+e^z)$.

Finding the maximum likelihood estimator in the above formulation is equivalent to minimizing the negative log-likelihood function $\l(\th)$,
\eq{\label{eq::neg-loglike}
\l(\th) = \frac{1}{n}\sum_{i=1}^n\left [  \phi (\<x_i,\th\>)-y_i\<x_i , \th \> \right ],
}
where $\<x , \th \>$ is the inner product between the vectors $x$ and $\th$. 
The relation to OLS and LR can be seen much easier by plugging in the corresponding $\phi(z)$ in Eq.~(\ref{eq::neg-loglike}).
The gradient and the Hessian of $\l(\th)$ can be written as:
\eq{
\grad_\th \l (\th) = \frac{1}{n}\sum_{i=1}^n\left [  \dphi (\<x_i,\th\>)x_i - y_ix_i \right ], \ \ \ \ \ 
\gradS_\th \l (\th) = \frac{1}{n}\sum_{i=1}^n \ddphi (\<x_i,\th\>)x_i x_i^T.
}
For a sequence of scaling matrices $\{\Q^t\}_{t>0} \in \reals^{p\times p}$, we consider iterations of the form
\eqn{
\hth^{t+1} \leftarrow \hth^{t} - \gamma_t \Q^t \grad_\th \l(\hth^{t})
}
where $\gamma_t$ is the step size. 
The above iteration is our main focus, but with a new approach on how to compute the sequence of matrices $\{\Q^t\}_{t>0}$. We will formulate the problem of finding a scalable $\Q^t$ as an estimation problem and apply a \emph{Stein}-type lemma that provides us with a computationally efficient update.


\section{Newton-Stein Method}\label{sec::algorithm}

Classical Newton-Raphson update is generally used for training GLMs. However, its per-iteration cost makes it impractical for large-scale optimization. The main bottleneck is the computation of the Hessian matrix that requires $\O(np^2)$ flops which is prohibitive when $n \gg p$. Numerous methods have been proposed to achieve NM's fast convergence rate while keeping the per-iteration cost manageable. 
To this end, a popular approach is to construct a scaling matrix $\Q^t$, which approximates the true Hessian at every iteration $t$.

The task of constructing an approximate Hessian can be viewed as an estimation problem. 
Assuming that the rows of $\X$ are i.i.d. random vectors, the Hessian of GLMs with cumulant generating function $\phi$ has the following form

\eqn{
\left[\Q^{t}\right]^{-1} = \frac{1}{n}\sum_{i=1}^nx_ix_i^T\ddphi(\<x_i,\th\>)\approx \E[xx^T\ddphi(\<x,\th\>)]\,.
}
We observe that $\left[\Q^{t}\right]^{-1}$ is just a sum of i.i.d. matrices. Hence, the true Hessian is nothing but a sample mean estimator to its expectation. Another natural estimator would be the sub-sampled Hessian method suggested by \cite{martens2010deep,byrd2011use}.
Therefore, our goal is to
propose an estimator that is computationally efficient and well-approximates the true Hessian. 

\begin{algorithm*}[t]
\caption{\ALGv}
\begin{algorithmic}
    \STATE {\bfseries Input:} $\hth^0, r, \tol, \gamma$.
    	\begin{enumerate}
	\STATE Set $t=0$ and sub-sample a set of indices $S\subset [n]$ uniformly at random.
        \STATE {\bfseries Compute:} $\hsig^2 = \lambda_{r+1}(\SigH_S )$, \ \ and \ \
         \hspace{.1in}$\zeta_r(\SigH_S)= \hsig^2\I + \argmin_{\text{rank($M$) = $r$}} \big\| \SigH_S-\hsig^2\I - M \big\|_F$.
        \WHILE{$\big\|\hth^{t+1}- \hth^{t}\big\|_2 \leq \tol$}
        \vspace{.1in}
               \STATE $\hmu_2(\hth^t) = \frac{1}{n}\sum_{i=1}^n\ddphi(\<x_i,\hth^t\>), \ \ \ \ \ \ \ \ \hmu_4(\hth^t) = \frac{1}{n}\sum_{i=1}^n\ddddphi(\<x_i,\hth^t\>)$,
               \STATE
        \STATE ${\Q^t = \frac{1}{\hmu_2(\hth^t)}\left[\zeta_r(\SigH_S)^{-1} - \frac{\hth^t[\hth^t]^T}{\hmu_2(\hth^t)/\hmu_4(\hth^t) + \<\zeta_r(\SigH_S)\hth^t, \hth^t\>}\right] }$,
  \vspace{.1in}
        \STATE $\hth^{t+1} = \proj_{B_p(R)} \left(\hth^{t} - \gamma \Q^t \grad_\th \l(\hth^{t})\right)$,
     \vspace{.1in}
	\STATE $t \leftarrow t+1$.
         \ENDWHILE
         \end{enumerate}
         \STATE {\bfseries Output:} $\hth^{t}$.
\end{algorithmic}
\end{algorithm*}

We use the following Stein-type proposition to find a more efficient estimator to the expectation of the Hessian.
\begin{lemma}[Stein-type lemma]\label{lem::stein}
Assume that $x\sim \normal_p(0,\Sig)$ and $\th \in \reals^p$ is a constant vector.
Then for any function $f:\reals \to \reals$ that is twice \emph{``weakly"} differentiable, we have
\eq{\label{eq::stein}
\E \left[xx^Tf(\<x,\th\>)\right] = \E\left[f(\<x,\th\>)\right]\Sig + \E\left[f^{(2)}(\<x,\th\>)\right]\Sig\th\th^T\Sig\,.
}
\end{lemma}
The proof of Lemma \ref{lem::stein} is given in Appendix. The right hand side of Eq.(\ref{eq::stein}) is a rank-1 update to the first term. Hence, its inverse can be computed with $\O(p^2)$ cost.
Quantities that change at each iteration are the ones that depend on $\th$, i.e.,

\eqn{
\mu_2(\th) = \E[\ddphi(\<x,\th\>)] \ \ \ \text{and} \ \ \ \mu_4(\th) = \E[\ddddphi(\<x,\th\>)].
}
Note that $\mu_2(\th)$ and $\mu_4(\th)$ are scalar quantities and can be estimated by their corresponding sample means $\hmu_2(\th)$ and $\hmu_4(\th)$ (explicitly defined at Step 3 of Algorithm 1) respectively, with only $\O(np)$ computation.

\begin{figure*}[t]
\centering
  \includegraphics[width=3.1in]{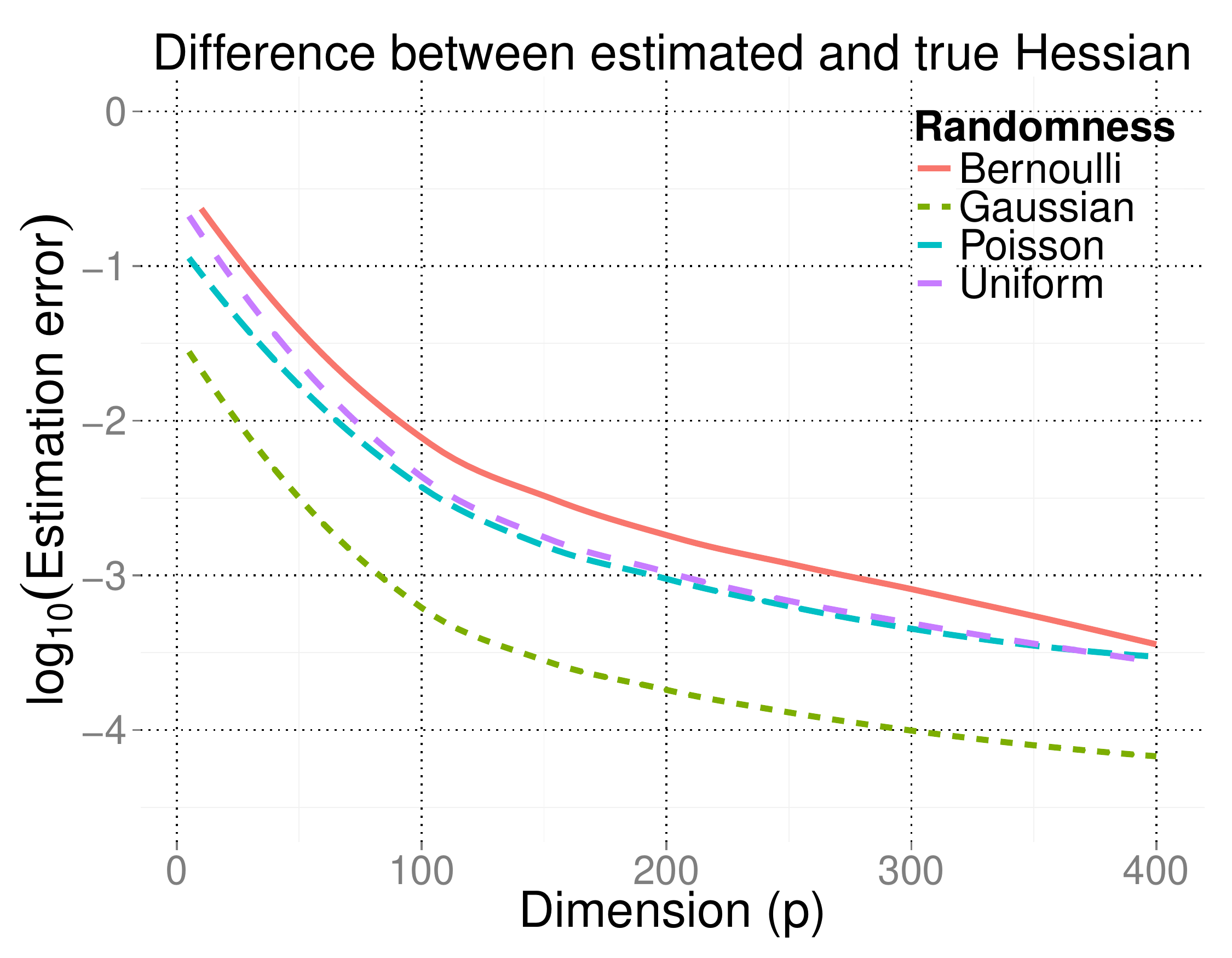}
\includegraphics[width=3.1in]{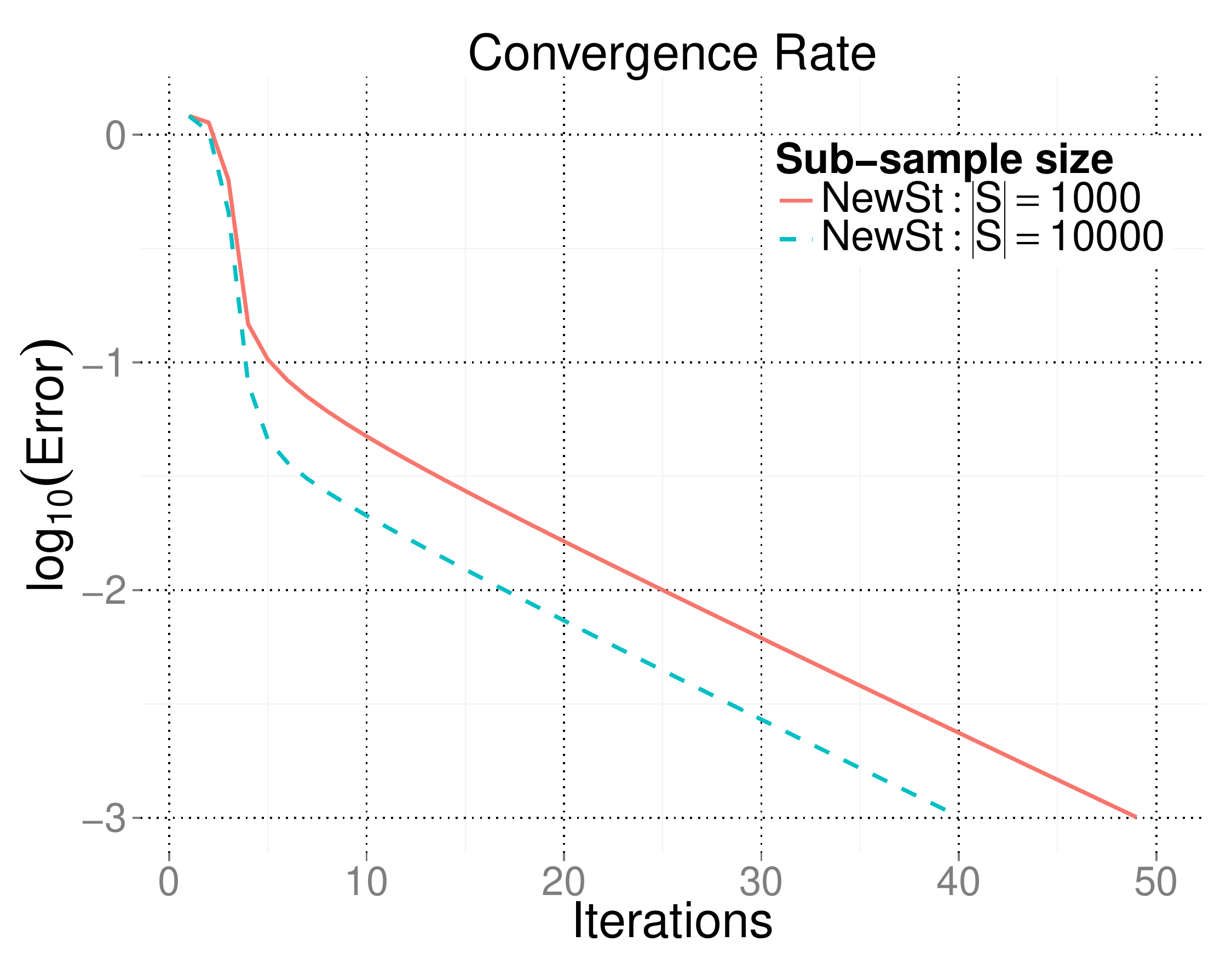}
  \caption{\label{fig::convAndCoeff}
The left plot demonstrates the accuracy of proposed Hessian estimation over different distributions. Number of observations is set to be $n=\O(p\log(p))$.
The right plot shows the phase transition in the convergence rate of \ALGv (\ALG\!).  Convergence starts with a quadratic rate and transitions into linear.
Plots are obtained using \emph{Covertype} dataset.
}
\end{figure*}

To complete the estimation task suggested by Eq.~(\ref{eq::stein}), we need an estimator for the covariance matrix $\Sig$.  A natural estimator is the sample mean where, we only use a sub-sample $S \subset [n]$ so that the cost is reduced to $\O(|S|p^2)$ from $\O(np^2)$. Sub-sampling based sample mean estimator is denoted by $\SigH_S=\sum_{i\in S}x_i x_i^T$, which is widely used in large-scale problems \cite{vershynin2010introduction}. 
We highlight the fact that Lemma \ref{lem::stein} replaces NM's $O(np^2)$ per-iteration cost with a one-time cost of $O(np^2)$. We further use sub-sampling to reduce this one-time cost to $O(|S|p^2)$.

In general, important curvature information is contained in the largest few spectral features \cite{erdogdu2015convergence-long}. 
For a given threshold $r$, we take the  largest $r$ eigenvalues of the sub-sampled covariance estimator, setting rest of them to $(r+1)$-th eigenvalue. This operation helps denoising and would only take $\O(rp^2)$ computation. Step 2 of Algorithm 1 performs this procedure.

Inverting the constructed Hessian estimator can make use of the low-rank structure several times. 
First, notice that the updates in Eq.~(\ref{eq::stein}) are based on rank-1 matrix additions. Hence, we can simply use a matrix inversion formula to derive an explicit equation (See $\Q^t$ in Step 3 of Algorithm 1). This formulation would impose another inverse operation on the covariance estimator. Since the covariance estimator is also based on rank-$r$ approximation, one can utilize the low-rank inversion formula again. We emphasize that this operation is performed once. Therefore, instead of NM's per-iteration cost of $\O(p^3)$ due to inversion, \ALGv (\ALG\!) requires $\O(p^2)$ per-iteration and a one-time cost of $\O(rp^2)$.
Assuming that \ALG and NM converge in $T_1$ and $T_2$ iterations respectively, 
the overall complexity of \ALG is $\O\left(npT_1+p^2T_1+(|S|+r)p^2\right)\approx\O\left(npT_1+p^2T_1+|S|p^2\right)$ whereas that of NM is $\O\left(np^2T_2 + p^3T_2\right)$.

Even though Proposition \ref{lem::stein} assumes that the covariates are multivariate Gaussian random vectors, in Section \ref{sec::theory}, the only assumption we make on the covariates is either bounded support or sub-gaussianity, both of which cover a wide class of random variables including Bernoulli, elliptical distributions, bounded variables etc. The left plot of Figure \ref{fig::convAndCoeff} shows that the estimation is accurate for many distributions. This is a consequence of the fact that the proposed estimator in Eq.~(\ref{eq::stein}) relies on the distribution of $x$ only through inner products of the form $\<x,v \>$, which in turn results in approximate normal distribution due to the central limit theorem. We will discuss this in details in Section \ref{sec::theory}.

The convergence rate of \ALG has two phases. Convergence starts quadratically and transitions into linear rate when it gets close to the true minimizer. The phase transition behavior can be observed through the right plot in Figure \ref{fig::convAndCoeff}. This is a consequence of the bound provided in Eq.~\ref{eq::boundMainResult}, which is the main result of our theorems stated in Section \ref{sec::theory}.


\section{Theoretical results}\label{sec::theory}

We start by introducing the terms that will appear in the theorems. Then we will provide two technical results on bounded and sub-gaussian covariates. The proofs of the theorems are technical and provided in Appendix.

\subsection{Preliminaries}

Hessian estimation described in the previous section relies on a Gaussian approximation. 
For theoretical purposes, we use the following probability metric to quantify the gap between the distribution of $x_i$'s and that of a normal vector. 

\begin{definition}
Given a family of functions $\Hh$, and random vectors $x,y \in \reals^p$, 
for $\Hh$ and any $h \in \Hh$, define 
\eqn{
d_\Hh(x,y) = \sup_{h \in \Hh}d_h(x,y) \ \ \ \ \ \text{where } \ \ \ \ \ d_h(x,y) = \big |\E\left[ h(x)\right] - \E\left[h(y) \right]\big |.
}
\end{definition}
%

Many probability metrics can be expressed as above by choosing a suitable function class $\Hh$. Examples include \emph{Total Variation} (TV), \emph{Kolmogorov} and \emph{Wasserstein} metrics \cite{gibbs2002choosing,chen2010normal}. 
%
%
Based on the second and the fourth derivatives of the cumulant generating function, we define the following function classes:
\eqn{
&\mathcal{H}_1 = \left \{ h(x) = \ddphi(\< x,\th\>) :  \th \in B_p(R)\right \},\ \ \ \ \ \
\mathcal{H}_2 = \left \{ h(x) = \ddddphi(\< x,\th\>):  \th \in B_p(R)\right \},\\
&\ \ \ \ \ \ \ \ \ \ \ \ \ \ \ \ \ \ \ \ \  \mathcal{H}_3 = \left \{h(x) = \<v,x\>^2 \ddphi(\< x,\th\>) :  \th \in B_p(R) , \| v \|_2 =1 \right \},
}
where $B_p(R)\in \reals^p$ is the ball of radius $R$.
Exact calculation of such probability metrics are often difficult. 
The general approach is to upper bound the distance by a more intuitive metric. 
In our case,  we observe that $d_{\Hh_j}(x,y)$ for $j=1,2,3$, 
can be easily upper bounded by $d_{\text{TV}}(x,y)$ 
up to a scaling constant, when the covariates have bounded support.

 We will further assume that the covariance matrix follows $r$-spiked model, i.e.,
$$
 \Sig =  \sigma^2I + \sum_{i=1}^r\theta_i u_i u_i^T,
 $$
which is commonly encountered in practice \cite{baik2006eigenvalues}.
The first $r$ eigenvalues of the covariance matrix are large and the rest are small and equal to each other. Small eigenvalues of $\Sig$ (denoted by $\sigma^2$), can be thought of as the noise component. 

\subsection{Bounded covariates}\label{sec::bounded}

We have the following per-step bound for the iterates generated by \ALG\!, when the covariates are supported on a ball. 
\begin{theorem}\label{thm::mainBounded}
Assume that the covariates $x_1,x_2,...,x_n$ are i.i.d. random vectors supported on a ball of radius $\sqrt{K}$ with
\eqn{
\E[x_i]= 0 \ \ \ \ \ \ \text{and} \ \ \ \ \ \ \ \E\big[x_ix_i^T\big]=\Sig,
}
where $\Sigma$ follows the $r$-spiked model.
Further assume that the cumulant generating function $\phi$ has bounded 2nd-5th derivatives and that $R$ is the radius of the projection $\proj_{B_p(R)}$.
For $\big\{\hth^t\big\}_{t>0}$ given by the Newton-Stein method for $\gamma=1$,
define the event 
\eq{\label{eq::stepBound1}
\mathcal{E} = \left\{\left | \mu_2(\hth^t) + \mu_4(\hth^t)\<\Sig\hth^t, \hth^t\>\right | >\xi\, , \ \ 
\th_* \in B_p(R) \right\}
}
for some positive constant $\xi$, and the optimal value $\th_*$. 
If $n,|S|$ and $p$ are sufficiently large, 
then there exist constants $c,c_1,c_2$ and $\kappa$
depending on
the radii $K,R$, $\P(\mathcal{E})$
and the bounds on $|\ddphi|$ and $|\ddddphi|$ such that 
conditioned on the event $\mathcal{E}$, with probability at least $1-c/p^2$, we have
%
\eq{\label{eq::composite}
\big\|\hth^{t+1} - \th_*\big\|_2
\leq \tau_1 \big \|\hth^t - \th_*\big\|_2  + \tau_2  \big\|\hth^t - \th_*\big\|_2^2  ,
}
where the coefficients $\tau_1$ and $\tau_2$ are deterministic constants defined as
\eq{
\tau_1 =  \kappa \D(x,z)
 +c_1\kappa \sqrt{\frac{p}{\min\left\{ p/\log(p)|S|,n/\log(n)\right\}}},
 \ \ \ \ \ \ \ \ \ \ \ \ \ 
\tau_2 =  c_2 \kappa,
}
and $\D(x,z)$ is defined as 
\eq{\label{eq::distance}
\D(x,z)=
\|\Sig\|_2 \ d_{\Hh_1}(x,z) + \|\Sig\|_2^2R^2\ d_{\Hh_2} (x,z)+d_{\Hh_3}(x,z),
}
for a multivariate Gaussian random variable $z$ with the same mean and covariance as $x_i$'s.
\end{theorem}

The bound in Eq.~(\ref{eq::composite}) holds with high probability, and the coefficients $\tau_1$ and $\tau_2$ are deterministic constants which will describe the convergence behavior of the Newton-Stein method. Observe that the coefficient $\tau_1$ is sum of two terms: $\D(x,z)$ measures how accurate the Hessian estimation is, and the second term depends on the sub-sample size and the data dimensions.

Theorem \ref{thm::mainBounded} shows that the convergence of Newton-Stein method can be upper bounded by a compositely converging sequence, that is, the squared term will dominate at first giving a quadratic rate, then the convergence will transition into a linear phase as the iterate gets close to the optimal value. The coefficients $\tau_1$ and $\tau_2$ govern the linear and quadratic terms, respectively. The effect of sub-sampling appears in the coefficient of linear term. In theory, there is a threshold for the sub-sampling size $|S|$, namely $\O( n/\log(n))$, beyond which further sub-sampling has no effect. The transition point between the quadratic and the linear phases is determined by the sub-sampling size and the properties of the data. The phase transition behavior can be observed through the right plot in Figure \ref{fig::convAndCoeff}. 


%
%
%
%
Using the above theorem, 
we state the following corollary.
\begin{corollary}\label{cor::compositeBounded}
Assume that the assumptions of Theorem \ref{thm::mainBounded} hold. For a constant $\delta \geq \P\left(\mathcal{E}^C\right)$, a tolerance $\e$ satisfying
\eqn{
\e \geq 20R\left\{c/p^2 + \delta\right\},
}
and for an iterate satisfying $\E\big[\|\hth^t-\th_*\|_2\big] >\e$, the iterates of Newton-Stein method will satisfy,
\eqn{
\E\left[\|\hth^{t+1} - \th_*\|_2 \right ]
\leq \tilde{\tau}_1 \E\left[\|\hth^t - \th_*\|_2 \right ]+ \tau_2 \E\left[\|\hth^t - \th_*\|_2^2 \right ],
}
where $\tilde{\tau}_1 = \tau_1+0.1$ and $\tau_1,\tau_2$ are as in Theorem \ref{thm::mainBounded}.
\end{corollary}
The bound stated in the above corollary is an analogue of composite convergence (given in Eq.~(\ref{eq::composite})) in expectation.
%
%
Note that our results make strong
assumptions on the derivatives of the cumulant generating function $\phi$. We emphasize that these assumptions are valid for linear and logistic regressions. An example that does not fit in our scheme is \emph{Poisson regression} with $\phi(z)=e^z$. However, we observed empirically that the algorithm still provides significant improvement.

The following theorem states a sufficient condition for the convergence of composite sequence.
\begin{theorem}\label{thm::globalNumIter}
Let $\{\hth^t\}_{t\geq 0}$ be a compositely converging sequence with convergence coefficients $\tau_1$ and $\tau_2$ as in Eq.~(\ref{eq::composite}) to the true minimizer $\th_*$. Let the starting point satisfy $\big\|\hth^0 - \th_*\big\|_2= \vartheta <(1-\tau_1)/\tau_2$ and define $\Xi = \left ( \frac{\tau_1 \vartheta }{1-\tau_2\vartheta} , \vartheta\right)$.
Then the sequence of $\ell_2$-distances converges to 0.
Further, the number of iterations to reach a tolerance of $\e$ 
can be upper bounded by $\inf_{\xi\in\Xi}\mathcal{J}(\xi)$, where
\eq{\label{eq::numIterations}
\mathcal{J}(\xi) = \log_2\left( \frac{\log\left(\delta\left(\tau_1/\xi + \tau_2\right)\right)}
{\log\left(\tau_1/\xi + \tau_2\right)\vartheta}\right)
+
\frac{\log(\e / \xi)}{\log(\tau_1 + \tau_2 \xi)}\, .
}
\end{theorem}
Above theorem gives an upper bound on the number of iterations until reaching a tolerance of $\e$. The first and second terms on the right hand side of Eq.~(\ref{eq::numIterations}) stem from the quadratic and linear phases, respectively. %
%
%
%
%
%
%

\subsection{Sub-gaussian covariates}
In this section, we carry our analysis to the more general case, where the covariates are sub-gaussian vectors. 
%
%
\begin{theorem}\label{thm::mainSG}
Assume that $x_1,x_2,...,x_n$ are i.i.d. sub-gaussian random vectors with sub-gaussian norm $K$ such that
\[
\E[x_i]= 0,  \ \ \ \ \ \ \ \ \ \E[\|x_i\|_2]= \mu \ \ \ \ \ \ \text{and} \ \ \ \ \ \ \ \E\big[x_ix_i^T\big]=\Sig,
\]
where $\Sig$ follows the $r$-spiked model.
Further assume that the cumulant generating function $\phi$ is uniformly bounded and has bounded 2nd-5th derivatives and that $R$ is the radius of the projection.
For $\big\{\hth^t\big\}_{t>0}$ given by \ALGv\!\! and the event $\mathcal{E}$ in Eq.~(\ref{eq::stepBound1}), if we have $n,|S|$ and $p$ sufficiently large and 
$$
n^{0.2}/\log(n)\gtrsim p,
$$
then
there exist constants $c_1,c_2,c_3,c_4$ and $\kappa$
depending on the eigenvalues of $\Sig$, 
the radius $R$, $\mu$, $\P(\mathcal{E})$
and the bounds on $\ddphi$ and $|\ddddphi|$ such that 
conditioned on the event $\mathcal{E}$, with probability at least $1-c_1e^{-c_2p}$, 
\eq{
\big\| \hth^{t+1} - \th_*\big\|_2 
\leq 
\kappa\left\{
\D(x,z)
 +c_3 \sqrt{\frac{p}{\min\left\{ |S|,n^{0.2}/\log(n)\right\}}}
+c_4 p^{1.5} \big\|\hth^t - \th_*\big\|_2 
  \right \}
  \big \| \hth^{t} - \th_*\big\|_2\nonumber
}
where $\D(x,z)$ defined as in Eq~\ref{eq::distance},
for a Gaussian random variable $z$ with the same mean and covariance as $x_i$'s.
\end{theorem}
The above theorem is more restrictive than Theorem \ref{thm::mainBounded}. We require $n$ to be much larger than the dimension $p$. Also note that a factor of $p^{1.5}$ appears in the coefficient of the quadratic term. We also notice that the threshold for sub-sample size reduces to $n^{0.2}/\log(n)$. 

We have the following analogue of Corrolary \ref{cor::compositeBounded}.

\begin{corollary}\label{cor::compositeSG}
Assume that the assumptions of Theorem \ref{thm::mainSG} hold. For a constant $\delta \geq \P\left(\mathcal{E}^C\right)$, a tolerance $\e$ satisfying
\eqn{
\e \geq 20R\sqrt{c_1e^{-c_2p} + \delta},
}
and for an iterate satisfying $\E\big[\|\hth^t-\th_*\|_2\big] >\e$, the iterates of Newton-Stein method will satisfy,
\eqn{
\E\left[\|\hth^{t+1} - \th_*\|_2 \right ]
\leq \tilde{\tau}_1 \E\left[\|\hth^t - \th_*\|_2 \right ]+ \tau_2 \E\left[\|\hth^t - \th_*\|_2^2 \right ],
}
where $\tilde{\tau}_1 = \tau_1+0.1$ and $\tau_1,\tau_2$ are as in Theorem \ref{thm::mainSG}.\end{corollary}
%
%

\subsection{Algorithm parameters}\label{sec::parameters}

\ALGv\! takes three input parameters
and for those, we suggest near-optimal choices based on our theoretical results.

\begin{itemize}
\item \textbf{Sub-sample size:} \ALG uses a subset of indices to approximate the covariance matrix $\Sig$. Corollary 5.50 of \cite{vershynin2010introduction} proves that a sample size of $\O(p)$ is sufficient for sub-gaussian covariates and that of $\O(p\log(p))$ is sufficient for arbitrary distributions supported in some ball to estimate a covariance matrix by its sample mean estimator. In the regime we consider, $n\gg p$, we suggest to use a sample size of $|S|=\O(p\log(p))$.

\item \textbf{Rank:} 
Many methods have been suggested to improve the estimation of covariance matrix and 
almost all of them rely on the concept of \emph{shrinkage}
\cite{cai2010singular,donoho2013optimal}.
Eigenvalue thresholding can be considered as a shrinkage operation which will retain only the important 
second order information. Choosing the rank threshold $r$ can be simply done on the sample mean estimator of $\Sig$. After obtaining the sub-sampled estimate of the mean, one can either plot the spectrum and choose manually or use an optimal technique from \cite{donoho2013optimal-hard}.

\item \textbf{Step size:}
Step size choices of \ALG are quite similar to Newton's method (i.e., See \cite{Boyd:2004}). The main difference comes from the eigenvalue thresholding. 
If the data follows the $r$-spiked model, the optimal step size will be close to 1 if there is no sub-sampling.
However, due to fluctuations resulting from sub-sampling, 
we suggest the following step size choice for \ALG:
\eq{\label{eq::optimalGamma}
\gamma= \frac{2}{1+\frac{\hsig^2 - \O(\sqrt{p/|S|})}{\hsig^2}}.
}
This formula yields a step size larger than 1. 
A detailed discussion can be found in Section \ref{sec::step-size} in Appendix.
\end{itemize}

\section{Experiments}\label{sec::experiments}

In this section, we validate the performance of \ALG through 
extensive numerical studies. 
We experimented on two commonly used GLM optimization problems, 
namely, \emph{Logistic Regression} (LR) and \emph{Linear Regression} (OLS).
LR minimizes Eq.~(\ref{eq::neg-loglike}) for the logistic function $\phi(z) = \log(1+e^z)$, whereas
OLS minimizes the same equation for $\phi(z) = z^2$.
In the following, we briefly describe the algorithms that are used in the experiments:
\begin{itemize}
\item \textbf{Newton's Method} (NM) uses the inverse Hessian evaluated at the current iterate, and may
achieve quadratic convergence. NM steps require $\O(np^2+p^3)$ computation which makes it impractical for large-scale datasets. 
\item \textbf{Broyden-Fletcher-Goldfarb-Shanno} (BFGS) forms a curvature matrix by cultivating the information from the iterates and the gradients at each iteration. Under certain assumptions,
the convergence rate is locally super-linear and the per-iteration cost is comparable to that of first order methods.
\item \textbf{Limited Memory BFGS} (L-BFGS) is similar to BFGS, and uses only the recent few iterates to construct the curvature matrix, gaining significant performance in terms of memory usage. 
\item \textbf{Gradient Descent} (GD) update is proportional to the negative of the full gradient evaluated at the current iterate. Under smoothness assumptions, GD achieves a linear convergence rate, with $\O(np)$ per-iteration cost.
\item \textbf{Accelerated Gradient Descent} (AGD) is proposed by Nesterov \cite{nesterov1983method}, which improves over the gradient descent by using a momentum term.
Performance of AGD strongly depends of the smoothness of the function.

%
%
\end{itemize}

For all the algorithms, we use a constant step size that provides the fastest convergence. 
Sub-sample size $|S|$, rank $r$ and the constant step-size $\gamma$ for \ALG is selected by following the guidelines described in Section \ref{sec::parameters}. The rank threshold $r$ (which is an input to the algorithm) is specified on the title each plot. 

\subsection{Simulations with synthetic sub-gaussian datasets}

Synthetic datasets, S3 and S20 are
generated through a multivariate Gaussian distribution where the covariance matrix follows $r$-spiked model, i.e., $r=3$ for S3 and $r=20$ for S20. 
To generate the covariance matrix, we first generate a random orthogonal matrix, say $\mathbf{M}$.
Next, we generate a diagonal matrix $\mathbf{\Lambda}$ that contains the eigenvalues, i.e., the first $r$ diagonal entries are
chosen to be large, and rest of them are equal to 1. Then, we let $\Sig = \mathbf{M}\mathbf{\Lambda} \mathbf{M}^T$. For dimensions of the datasets, see Table \ref{tab::datasets}.
We also emphasize that the
data dimensions are chosen so that Newton's method still does well.

\begin{figure*}[t]
\centering
 \includegraphics[width=5in]{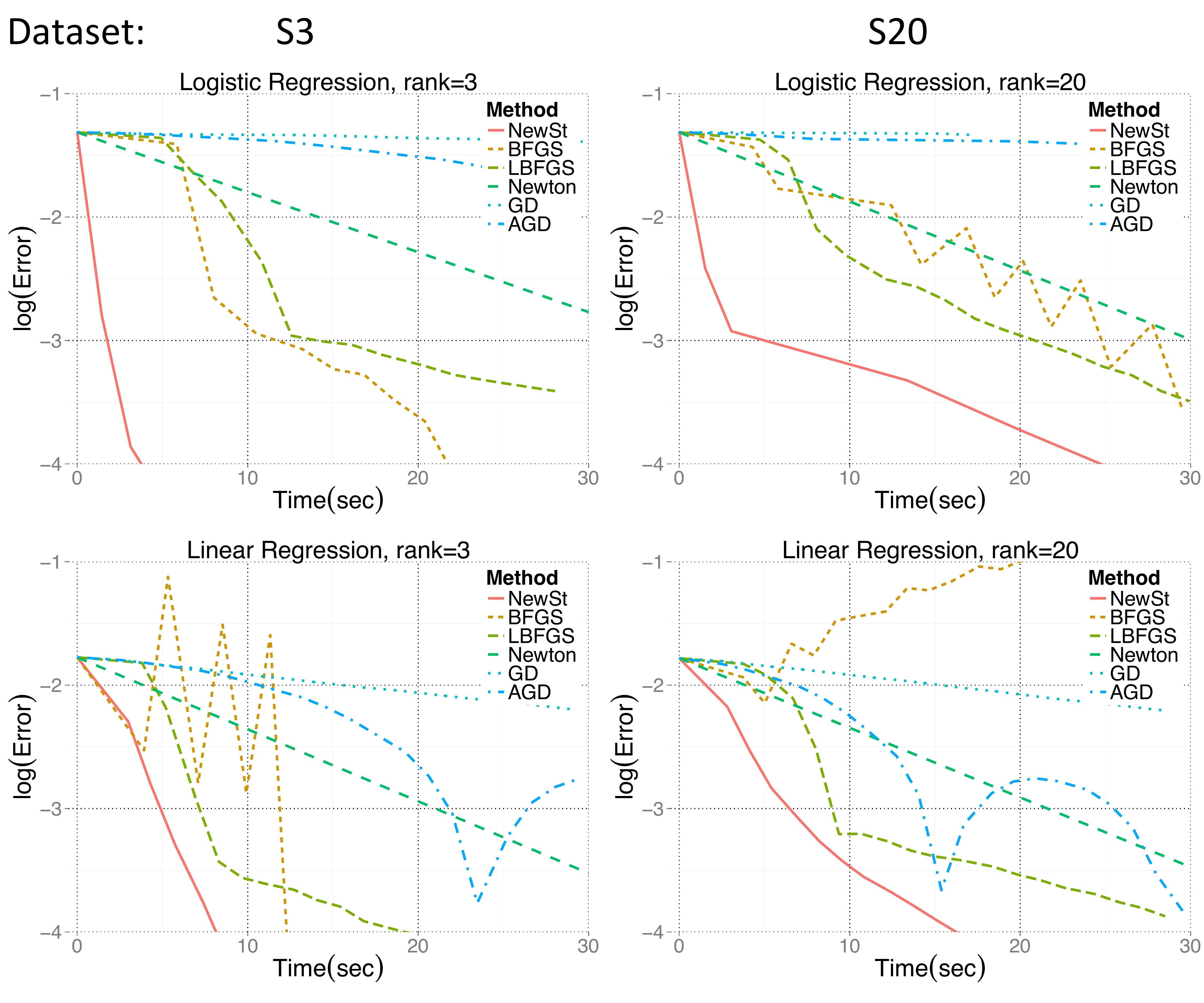}
  \caption{\label{fig::plotSyn}
Performance of various optimization methods on synthetic datasets.
Red straight line represents the proposed method \ALG\!. Algorithm parameters including the rank threshold is selected by the guidelines described in Section \ref{sec::parameters}.
}
\end{figure*}

The simulation results are summarized in Figure \ref{fig::plotSyn}. Further details regarding the experiments can be found in Table \ref{tab::details} in Appendix \ref{sec::expDetail}.  We observe that \ALG provides a significant improvement over the classical techniques. 

Observe that the convergence rate of \ALG has a clear phase transition point in the top left plot in Figure \ref{fig::plotSyn}. As argued earlier, this point depends on various factors including sub-sampling size $ |S|$ and data dimensions $n,p$, the rank threshold $r$ and structure of the covariance matrix. The prediction of the phase transition point is an interesting line of research. However, our convergence guarantees are conservative and cannot be used for this purpose.

\subsection{Experiments with Real datasets}
We experimented on two real datasets where the datasets are downloaded from UCI repository \cite{lichman2013}. Both datasets satisfy $n \gg p$, but we highlight the difference between the proportions of dimensions $n/p$. See Table \ref{tab::datasets} for details.

\begin{figure*}[t]
\centering
 \includegraphics[width=5in]{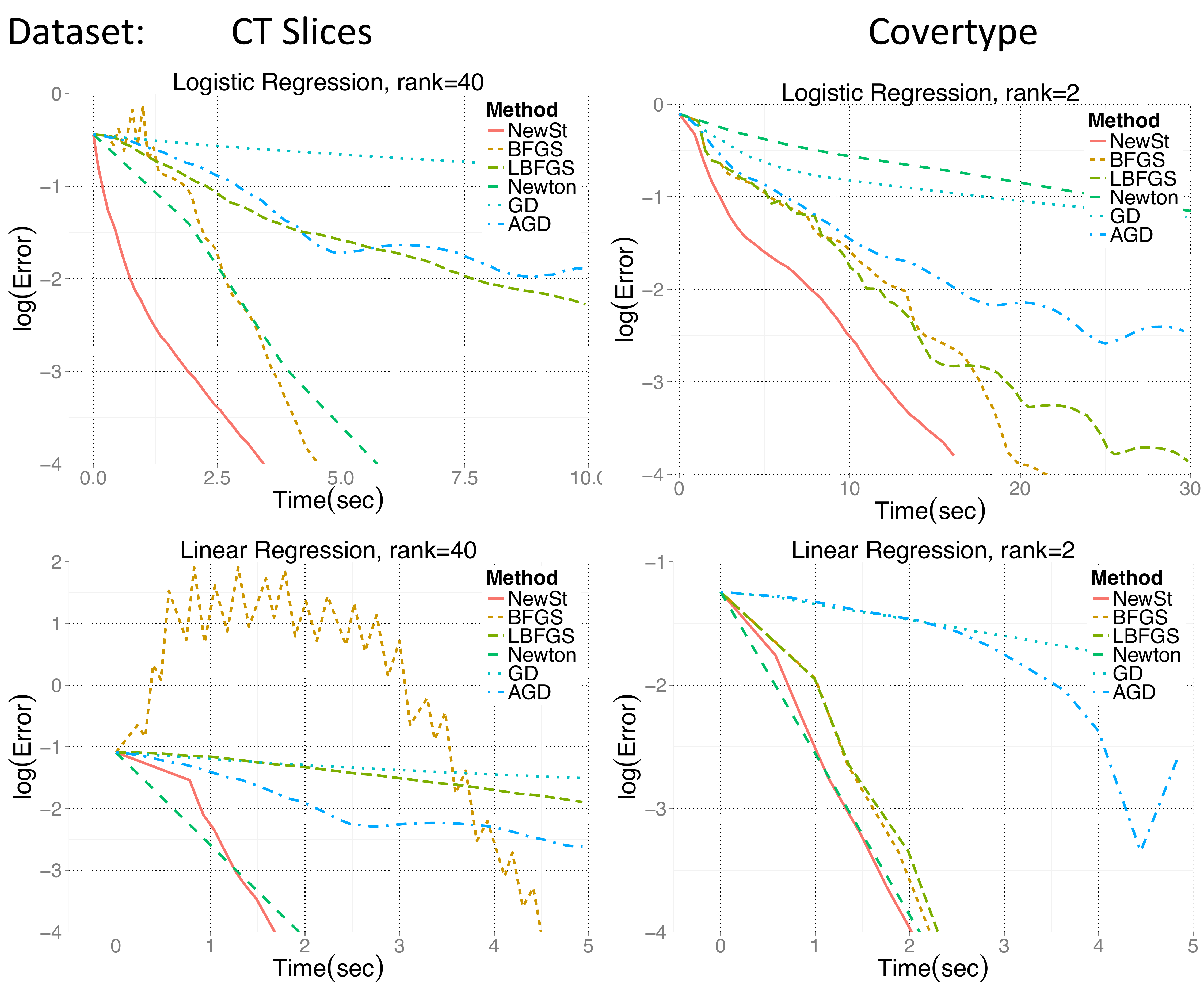}
  \caption{\label{fig::plotReal}
Performance of various optimization methods on real datasets.
Red straight line represents the proposed method \ALG\!. Algorithm parameters including the rank threshold is selected by the guidelines described in Section \ref{sec::parameters}.
}
\end{figure*}

We observe that \ALGv\!\! performs better than classical methods on real datasets as well. More specifically,
the methods that come closer to \ALG is Newton's method for moderate $n$ and $p$ and BFGS when $n$ is large. 

The optimal step-size for \ALG will typically be larger than 1 which is mainly due to eigenvalue thresholding operation. This feature is desirable if one is able to obtain a large step-size that provides convergence. In such cases, the convergence is likely to be faster, yet more unstable compared to the smaller step size choices.
We observed that similar to other second order algorithms, \ALG is also susceptible to the step size selection. If the data is not well-conditioned, and the sub-sample size is not sufficiently large, algorithm might have poor performance. 
This is mainly because the sub-sampling operation is performed only once at the beginning.
Therefore, it might be good in practice to sub-sample once in every few iterations.

\begin{table}[H]
\centering
\begin{tabular}{|l||l|l|l|}
\hline
Dataset    & $n$ & $p$ &   Reference, UCI repo \cite{lichman2013}\\
\hline
CT slices		&53500	&386		&\cite{graf20112d}\\
Covertype 	&581012 	&54		&\cite{blackard1999comparative}\\
S3  	&500000	&300		& 3-spiked model, \cite{donoho2013optimal}\\
S20  	&500000	&300		& 20-spiked model, \cite{donoho2013optimal}\\
\hline
\end{tabular}
\caption{\label{tab::datasets}
Datasets used in the experiments.
}
\end{table}

\section{Discussion}\label{sec::discussion}

In this paper, we proposed an efficient algorithm for training GLMs. We called our algorithm Newton-Stein Method (\ALG\!) as it takes a Newton update at each iteration relying on a Stein-type lemma. The algorithm requires a one time $\O(|S|p^2)$ cost to estimate the covariance structure and $\O(np)$ per-iteration cost to form the update equations. 
We observe that the convergence of \ALG has a phase transition from quadratic rate to linear. This observation is justified theoretically along with several other guarantees for the sub-gaussian covariates such as per-step convergence bounds, conditions for convergence, etc. Parameter selection guidelines of \ALG\! are based on our theoretical results. Our experiments show that \ALG\! provides significant improvement over several optimization methods.

Relaxing some of the theoretical constraints is an interesting line of research. In particular, strong assumptions on the cumulant generating functions might be loosened. 
Another interesting direction is to determine when the phase transition point occurs, 
which would provide a better understanding of the effects of sub-sampling and rank threshold.

\section*{Acknowledgements}
The author is grateful to Mohsen Bayati and Andrea Montanari for stimulating conversations on the topic of this work. The author would like to thank Bhaswar B. Bhattacharya and Qingyuan Zhao
for carefully reading this article and providing valuable feedback.

\bibliographystyle{amsalpha}
\bibliography{bib}

\newpage
\appendix


\section{Proof of Stein-type lemma}
\begin{proof}[Proof of Lemma \ref{lem::stein}]
The proof will follow from integration by parts over multivariate variables. 
Let $g(x)$ be the density of $x$, i.e.,
\eqn{
g(x) = (2\pi)^{-p/2}|\Sig|^{-1/2}\exp\left\{ -\frac{1}{2}\< \Sig^{-1}x,x\>\right\},
}
and $xg(x)dx = -\Sig dg(x)$.
We write
\eqn{
\E [xx^Tf(\<x,\th\>)] =& \int xx^Tf(\<x,\th\>) g(x)dx\\
=& \Sig \int -f(\<x,\th\>) dg(x)x^T,\\
=& \Sig\left \{ \int f(\<x,\th\>) g(x)dx + \int\th x^T\dot{f}(\<x,\th\>)g(x)dx \right \},\\
=& \Sig\left \{ \E [f(\<x,\th\>)] + \int\th \th^T\ddot{f}(\<x,\th\>)g(x)dx\Sig \right \},\\
=& \Sig\left \{ \E [f(\<x,\th\>)] + \th \th^T\E[\ddot{f}(\<x,\th\>)]\Sig \right \},\\
=& \E[f(\<x,\th\>)]\Sig + \E[\ddot{f}(\<x,\th\>)]\Sig\th\th^T\Sig ,
}
which is the desired result.
\end{proof}

\section{Preliminary concentration inequalities}
In this section, we provide concentration results that will be useful proving the main theorem. We start with some simple definitions on a special class of random variables.

\begin{definition}[Sub-gaussian]
For a given constant $K$, a random variable $x\in \reals$ is called \emph{sub-gaussian} if it satisfies
\[
\E [|x|^m]^{1/m} \leq K \sqrt{m}, \ \ \ \   m\geq1.
\]
Smallest such $K$ is the sub-gaussian norm of $x$ and it is denoted by $\|x\|_\sg$. Similarly, a random vector $y \in \reals^p$ is a \emph{sub-gaussian vector} if there exists a constant $K'>0$
such that
\eqn{
\sup_{v \in S^{p-1}}\| \< y, v\>\|_\sg \leq K'.
}
\end{definition}
\begin{definition}[Sub-exponential]
For a given constant $K$, a random variable $x\in \reals$  is called \emph{sub-exponential} if it satisfies
\[
\E [|x|^m]^{1/m} \leq K {m},\ \ \ \   m\geq1,
\]
Smallest such $K$ is the sub-exponential norm of $x$ and it is denoted by $\|x\|_\se$. Similarly, a random vector $y \in \reals^p$ is a \emph{sub-exponential vector} if there exists a constant $K'>0$
such that
\eqn{
\sup_{v \in S^{p-1}}\| \< y, v\>\|_\se \leq K'.
}
\end{definition}

We state the following Lemmas from \cite{vershynin2010introduction} for the convenience of the reader (i.e., See Theorem 5.39 and the following remark for sub-gaussian distributions, and Theorem 5.44 for distributions with arbitrary support):
\begin{lemma}[\cite{vershynin2010introduction}]\label{lem::concentration1}
Let $S$ be an index set and $x_i\in \reals^p$ for $i\in S$ be i.i.d. sub-gaussian random vectors with 
\eqn{
\E[x_i] = 0, \ \ \ \ \ \ \ \E[x_ix_i^T] = \Sig \ \ \ \ \ \  \| x_i\|_\sg \leq K.
}
There exists absolute constants $c,C$ depending only on the sub-gaussian norm $K$ such that with probability $1-2e^{-ct^2}$,
$$\left\|\SigH_S -\Sig\right\|_2 \leq \max\left(\delta,\delta^2\right)\ \ \ \text{where}\ \ \ \delta = C\sqrt{\frac{p}{|S|}} + \frac{t}{\sqrt{|S|}}.$$
\end{lemma}
\begin{remark}
We are interested in the case where $\delta<1$, hence the right hand side becomes $\max\left(\delta,\delta^2\right)=\delta$. In most cases, we will simply let $t=\sqrt{p}$ and obtain a bound of order $\sqrt{p/|S|}$ on the right hand side. For this, we need $|S| = \O(C^2p)$ which is a reasonable assumption in the regime we consider. 
\end{remark}

The following lemma will be helpful to show a similar concentration result for the matrix $\zeta_r(\SigH_S)$:

\begin{lemma}\label{lem::concentration2}
Let the assumptions in Lemma \ref{lem::concentration1} hold. Further, assume that $\Sig$ follows $r$-spiked model. If $|S|$ is sufficiently large, then there exists absolute constants $c,C$ depending only on the sub-gaussian norm $K$ such that with probability $1-2e^{-cp}$,
$$\left\|\zeta_r(\SigH_S)-\SigH_S\right\|_2 \leq C\sqrt{\frac{p}{|S|}}.$$
\end{lemma}
\begin{proof}
By the Weyl's inequality for the eigenvalues, we have
\eqn{
\left\|\zeta_r(\SigH_S)-\SigH_S\right\|_2 =& \lambda_{r+1}(\SigH_S) - \lambda_p(\SigH_S)
\leq 2 \|\SigH_S - \Sig\|_2.
}
Hence the result follows from the previous lemma for $t=\sqrt{p}$.
\end{proof}
Lemmas \ref{lem::concentration1} and \ref{lem::concentration2} are straightforward concentration results for the random matrices with i.i.d.  sub-gaussian rows. We state their analogues for the the covariates sampled from arbitrary distributions with bounded support.

\begin{lemma}[\cite{vershynin2010introduction}]\label{lem::concentration1bounded}
Let $S$ be an index set and $x_i\in \reals^p$ for $i\in S$ be i.i.d. random vectors with 
\eqn{
\E[x_i] = 0, \ \ \ \ \ \ \ \E[x_ix_i^T] = \Sig, \ \ \ \ \ \  \| x_i\|_2 \leq \sqrt{K}\ \text{a.s.}
}
Then, for some absolute constant $C$,
with probability $1-pe^{-Ct^2}$, we have
$$\big\|\SigH_S -\Sig\big\|_2 \leq \max\left(\|\Sig\|_2^{1/2}\delta,\delta^2\right)\ \ \ \text{where}\ \ \ \delta = t\sqrt{\frac{K}{|S|}}.$$
\end{lemma}

\begin{remark}\label{rem::concentration2}
We will choose $t=\sqrt{3\log(p)/C}$ which will provide us with a probability of $1-1/p^2$. Therefore, if the sample size is sufficiently large, i.e., $$|S| \geq \frac{3K\log(p)}{C\|\Sig\|_2} = \O(K\log(p)/\|\Sig\|_2),$$
we can estimate the true covariance matrix quite well for arbitrary distributions with bounded support. In particular, with probability $1-1/p^2$, we obtain
\eqn{
\big\|\SigH_S -\Sig\big\|_2 \leq c \sqrt{\frac{\log(p)}{|S|}},
}
where $c=\sqrt{3K\|\Sig\|_2/C}$.
\end{remark}

\begin{lemma}\label{lem::concentration2bounded}
Let the assumptions in Lemma \ref{lem::concentration1bounded} hold. Further, assume that $\Sig$ follows $r$-spiked model. If $|S|$ is sufficiently large, for $c=2\sqrt{3K\|\Sig\|_2/C}$, with probability $1-1/p^2$, we have
$$\left\|\zeta_r(\SigH_S)-\SigH_S\right\|_2 \leq c\sqrt{\frac{\log(p)}{|S|}},$$
where $C$ is an absolute constant.
\end{lemma}
Proof of Lemma \ref{lem::concentration2bounded} is the same as that of Lemma \ref{lem::concentration2}.
Before proceeding, we note that the bounds given in Lemmas \ref{lem::concentration2} and \ref{lem::concentration2bounded} also applies to $\big\|\zeta_r(\SigH_S)-\Sig\big\|_2$.

In the following, we will focus on the empirical processes and obtain more technical bounds for the approximate Hessian. To that extent, we provide some basic definitions that will be useful later in the proofs. For a more detailed discussion on the machinery used throughout next section, we refer interested reader to \cite{van2000asymptotic}.\\

\begin{definition}
On a metric space $(X,d)$, for $\epsilon >0$, $T_\epsilon \subset X$ is called an $\epsilon$-net over $X$ if $\forall x \in X$, $\exists t \in T_\epsilon$ such that $d(x,t)\leq \epsilon$.
\end{definition}

In the following, we will use $L_1$ distance between two functions $f$ and $g$, namely $d(f,g) = \int |f-g|$. Note that the same distance definition can be carried to random variables as they are simply real measurable functions.
\begin{definition}
Given a function class $\F$, and any two functions $l$ and $u$ (not necessarily in $\F$), the bracket $[l,u]$ is the set of all $f \in \F$ such that $l \leq f \leq u$. A bracket satisfying $l\leq u$ and $\int |u-l| \leq \epsilon$ is called an $\epsilon$-bracket in $L_1$. The bracketing number $\N (\epsilon, \F, L_1)$ is the minimum number of different $\epsilon$-brackets needed to cover $\F$.
\end{definition}

The preliminary tools presented in this section will be utilized to 
obtain the concentration results in Section \ref{sec::main-lemmas}.


\section{Main lemmas}\label{sec::main-lemmas}
\subsection{Concentration of covariates with bounded support}

\begin{lemma}\label{lem::concentrationBounded1}
Let $x_i\in \reals^p$, for $i=1,2,...,n$, be i.i.d. random vectors supported on a ball of radius $\sqrt{K}$, with mean $0$, and covariance matrix $\Sig$. Also let $f:\reals \to \reals$ be a uniformly bounded function such that for some $B >0 $, we have $\|f \|_\infty < B$ and $f$ is Lipschitz continuous with constant $L$.  Then, there exist constants $c_1,c_2,c_3$ such that
\eqn{
\P \left( \sup_{\th \in B_n(R)} \left | \frac{1}{n}\sum_{i=1}^n f(\<x_i,\th \>) - \E[f(\< x,\th\>)]  \right | > c_1\sqrt{\frac{ p\log(n)}{n}} \right) \leq c_2e^{-c_3p},
}
where the constants depend only on the bound $B$ and radii $R$ and $\sqrt{K}$.
\end{lemma}
\begin{proof}[Proof of Lemma \ref{lem::concentrationBounded1}]
We start by using the Lipschitz property of the function $f$, i.e., $\forall \th, \th' \in B_p(R)$, 
\eqn{
\| f(\<x,\th\>)-f(\<x,\th'\>)\|_2 \leq& L \|x\|_2 \|\th-\th'\|_2,\\
\leq&  L \sqrt{K}\|\th-\th'\|_2.
}

Now let $T_\Delta$ be a $\Delta$-net over $B_p(R)$. Then $\forall \th \in B_p(R)$, $\exists \th'\in T_\Delta$ such that right hand side of the above inequality is smaller than $\Delta L\sqrt{K}$.
Then, we can write
\eq{\label{eq::supMaxBound}
 \left | \frac{1}{n}\sum_{i=1}^n f(\<x_i,\th \>) - \E[f(\< x,\th\>)]  \right |
 \leq  \left | \frac{1}{n}\sum_{i=1}^n f(\<x_i,\th' \>) - \E[f(\< x,\th'\>)]  \right | + 2\Delta L\sqrt{K}.
}

By choosing $$\Delta = \frac{\epsilon}{4L\sqrt{K}},$$
and taking supremum over the corresponding $\th$ sets on both sides, we obtain the following inequality
\eqn{
 \sup_{\th \in B_n(R)} \left | \frac{1}{n}\sum_{i=1}^n f(\<x_i,\th \>) - \E[f(\< x,\th\>)]  \right | 
 \leq 
 \max_{\th \in \T_\Delta} \left | \frac{1}{n}\sum_{i=1}^n f(\<x_i,\th \>) - \E[f(\< x,\th\>)]  \right | +\frac{\e}{2}.
}

Now, since we have $\|f\|_\infty \leq B$ and for a fixed $\th$ and $i=1,2,...,n$, the random variables $f(\<x_i,\th\>)$ are i.i.d., by the Hoeffding's concentration inequality, we have
\eqn{
\P\left(\left | \frac{1}{n}\sum_{i=1}^n f(\<x_i,\th \>) - \E[f(\< x,\th\>)]  \right |>\e/2\right)
\leq 2\exp\left( -\frac{n\e^2}{2B^2}\right).
}
Combining Eq.~(\ref{eq::supMaxBound}), the above result together with the union bound, we easily obtain
\eqn{
\P\left( \sup_{\th \in B_n(R)}\left | \frac{1}{n}\sum_{i=1}^n f(\<x_i,\th \>) - \E[f(\< x,\th\>)]  \right | > \e\right)\leq&
\P\left(\max_{\th \in \T_\Delta} \left | \frac{1}{n}\sum_{i=1}^n f(\<x_i,\th \>) - \E[f(\< x,\th\>)]  \right | > \e/2\right) \\\leq&
2|\T_\Delta|\exp\left( -\frac{n\e^2}{2B^2}\right),
}
where $\Delta = {\epsilon}/{4L\sqrt{K}}$.

Next, we apply Lemma \ref{lem::sphere} and obtain that
$$|\T_\Delta| \leq \left(\frac{R\sqrt{p}}{\Delta}\right)^p = \left(\frac{R\sqrt{p}}{ {\epsilon}/{4L\sqrt{K}}}\right)^p. $$

We require that the bound on the probability gets an exponential decay with rate $\O(p)$. Using Lemma \ref{lem::epsilon} with $a=2B^2p/n$ and $b = 4LR\sqrt{Kp}$, we obtain that $\e$ should be
\eqn{
\e = \sqrt{
\frac{B^2p}{n}
\log\left( \frac{16L^2R^2Kn}{B^2}\right)
}
= \O\left(\sqrt{\frac{p\log(n)}{n}} \right),
}
which completes the proof.
\end{proof}

In the following, we state the concentration results on functions of the form 
$$x\to f(\<x,\th \>)\<x,v\>^2.$$ Functions of this type form the summands of the Hessian matrix in GLMs.
\begin{lemma}\label{lem::concentrationBounded2}
Let $x_i\in \reals^p$, for $i=1,...,n$, be i.i.d. random vectors supported on a ball of radius $\sqrt{K}$, with mean $0$, covariance matrix $\Sig$. Also let $f:\reals \to \reals$ be a uniformly bounded function such that for some $B >0 $, we have $\|f \|_\infty < B$ and $f$ is Lipschitz continuous with constant $L$.  Then, for $v \in S^{p-1}$, there exist constants $c_1,c_2,c_3$ such that
\eqn{
\P \left( \sup_{\th \in B_p(R)} \left | \frac{1}{n}\sum_{i=1}^n f(\<x_i,\th \>)\<x_i,v\>^2 - \E[f(\< x,\th\>)\<x,v\>^2]  \right | > c_1\sqrt{\frac{p\log\left(n \right)}{n} } \right) \leq  c_2e^{-c_3p} ,\\
}
where the constants depend only on the bound $B$ and radii $R$ and $\sqrt{K}$.
\end{lemma}

\begin{proof}[Proof of Lemma \ref{lem::concentrationBounded2}]
As in the proof of Lemma \ref{lem::concentrationBounded1}, we start by using the Lipschitz property of the function $f$, i.e., $\forall \th,\th'\in B_p(R)$,
\eqn{
\| f(\<x,\th\>)\<x,v\>^2-f(\<x,\th'\>)\<x,v\>^2\|_2 \leq& L \|x\|^3_2 \|\th-\th'\|_2,\\
\leq&  L K^{1.5}\|\th-\th'\|_2.
}

For a net $T_\Delta$, $\forall \th \in B_p(R)$, $\exists \th'\in T_\Delta$ such that right hand side of the above inequality is smaller than $\Delta L\sqrt{K}$.
Then, we can write
\eq{\label{eq::supMaxBound2}
 \left | \frac{1}{n}\sum_{i=1}^n f(\<x_i,\th \>)\<x_i,v\>^2 - \E[f(\< x,\th\>)\<x,v\>^2]  \right |\nonumber
 \leq&  \left | \frac{1}{n}\sum_{i=1}^n f(\<x_i,\th' \>)\<x_i,v\>^2 - \E[f(\< x,\th'\>)\<x,v\>^2]  \right | \\
 &+ 2\Delta LK^{1.5}.
}

This time, we choose $$\Delta = \frac{\epsilon}{4LK^{1.5}},$$
and take the supremum over the corresponding feasible $\th$-sets on both sides,
\eqn{
& \sup_{\th \in B_n(R)} \left | \frac{1}{n}\sum_{i=1}^n f(\<x_i,\th \>)\<x_i,v\>^2 - \E[f(\< x,\th\>)\<x,v\>^2]  \right | \\
 &\leq 
 \max_{\th \in \T_\Delta} \left | \frac{1}{n}\sum_{i=1}^n f(\<x_i,\th \>)\<x_i,v\>^2 - \E[f(\< x,\th\>)\<x,v\>^2]  \right | +\frac{\e}{2}.
}

Now, since we have $\|f\|_\infty \leq B$ and for fixed $\th$ and $v$, $i=1,2,...,n$, $f(\<x_i,\th\>)\<x_i,v\>^2$ are i.i.d. random variables. By the Hoeffding's concentration inequality, we have
\eqn{
\P\left(\left | \frac{1}{n}\sum_{i=1}^n f(\<x_i,\th \>)\<x_i,v\>^2 - \E[f(\< x,\th\>)\<x,v\>^2]  \right |>\e/2\right)
\leq 2\exp\left( -\frac{n\e^2}{2B^2K^2}\right).
}
Using Eq.~(\ref{eq::supMaxBound2}) and the above result combined with the union bound, we easily obtain
\eqn{
&\P\left( \sup_{\th \in B_n(R)} \left | \frac{1}{n}\sum_{i=1}^n f(\<x_i,\th \>)\<x_i,v\>^2 - \E[f(\< x,\th\>)\<x,v\>^2]  \right | > \e\right)\\
&\leq
\P\left(\max_{\th \in \T_\Delta} \left | \frac{1}{n}\sum_{i=1}^n f(\<x_i,\th \>)\<x_i,v\>^2 - \E[f(\< x,\th\>)\<x,v\>^2]  \right | > \e/2\right) \\
&\leq
2|\T_\Delta|\exp\left( -\frac{n\e^2}{2B^2K^2}\right),
}
where $\Delta = {\epsilon}/{4LK^{1.5}}$.
Using Lemma \ref{lem::sphere}, we have
$$|\T_\Delta| \leq \left(\frac{R\sqrt{p}}{\Delta}\right)^p = \left(\frac{R\sqrt{p}}{ {\epsilon}/{4LK^{1.5}}}\right)^p. $$

As before, we require that the right  hand side of above inequality gets a decay with rate $\O(p)$. Using Lemma \ref{lem::epsilon} with $a=2B^2K^2p/n$ and $b = 4LRK^{1.5}\sqrt{p}$, we obtain that $\e$ should be
\eqn{
\e = \sqrt{
\frac{B^2K^2p}{n}
\log\left( \frac{16L^2R^2K^3n}{B^2}\right)
}
=\O\left(\sqrt{\frac{p\log(n)}{n}} \right),
}
which completes the proof.
\end{proof}

\subsection{Concentration of sub-gaussian covariates}
In this section, we derive the analogues of the Lemmas \ref{lem::concentrationBounded1} and \ref{lem::concentrationBounded2} for sub-gaussian covariates. Note that the Lemmas in this section are more general in the sense that they also apply to the case where covariates have bounded support. However, the concentration coefficients are different (worse) compared to previous section.

\begin{lemma}\label{lem::concentration3}
Let $x_i\in \reals^p$, for $i=1,...,n$, be i.i.d. sub-gaussian random vectors with mean $0$, covariance matrix $\Sig$ and sub-gaussian norm $K$. Also let $f:\reals \to \reals$ be a uniformly bounded function such that for some $B >0 $, we have $\|f \|_\infty < B$ and $f$ is Lipschitz continuous with constant $L$.  Then, there exists absolute constants $c_1,c_2,c_3$ such that
\eqn{
\P \left( \sup_{\th \in B_n(R)} \left | \frac{1}{n}\sum_{i=1}^n f(\<x_i,\th \>) - \E[f(\< x,\th\>)]  \right | > c_1\sqrt{\frac{ p\log(n)}{n}} \right) \leq c_2e^{-c_3p},
}
where the constants depend only on the eigenvalues of $\Sig$, bound $B$ and radius $R$ and sub-gaussian norm $K$.
\end{lemma}

\begin{proof}[Proof of Lemma \ref{lem::concentration3}]
We start by defining the brackets of the form
\eqn{
l_\th(x) =& f(\< x,\th\>) - \epsilon\frac{\| x\|_2}{4\E\left[\| x\|_2\right]},\\
u_\th(x)=& f(\< x,\th\>) + \epsilon\frac{\| x\|_2}{4\E\left[\| x\|_2\right]}.
}
Observe that the size of bracket $[\l_\th , u_\th]$ is $\epsilon/2$.
Now let $T_\Delta$ be a $\Delta$-net over $B_p(R)$ where we use $\Delta = \epsilon /(4L \E\left[\| x\|_2\right])$. Then $\forall \th \in B_p(R)$, $\exists \th'\in T_\Delta$ such that $f(\< \cdot , \th \>)$ belongs to the bracket $[\l_{\th'} , u_{\th'}]$. This can be seen by writing out the Lipschitz property of the function $f$. That is,
\eqn{
\| f(\<x,\th\>)-f(\<x,\th'\>)\|_2 \leq& L \|x\|_2 \|\th-\th'\|_2,\\
\leq& \Delta L \|x\|_2,
}
where the first inequality follows from Cauchy-Schwartz.
Therefore, we conclude that
\eqn{
\N (\epsilon/2, \F, L_1) \leq |T_\Delta|
}
for the function class $\F = \{f(\<\cdot,\th\>) : \th \in B_p(R) \}$.
We further have $\forall \th \in B_p(R)$, $\exists \th' \in T_\Delta$ such that
\eqn{
 \frac{1}{n}\sum_{i=1}^n f(\<x_i,\th \>) - \E[f(\< x,\th\>)] \leq& \  \frac{1}{n}\sum_{i=1}^n u_{\th'}(x_i) - \E[u_{\th'}(x)] +\frac{\epsilon}{2}, \\
  \frac{1}{n}\sum_{i=1}^n f(\<x_i,\th \>) - \E[f(\< x,\th\>)] \geq& \ \frac{1}{n}\sum_{i=1}^n l_{\th'}(x_i) - \E[l_{\th'}(x)] -\frac{\epsilon}{2}.
}
Using the above inequalities, we have, $\forall \th \in B_p(R)$, $\exists \th' \in T_\Delta$
\eqn{
 &\left\{\left [ \frac{1}{n}\sum_{i=1}^n u_{\th'}(x_i) - \E[u_{\th'}(x)]  \right ] > \e/2 \right\} \cup
 \left\{\left [ -\frac{1}{n}\sum_{i=1}^n l_{\th'}(x_i) + \E[l_{\th'}(x)]  \right ] > \e/2 \right\}  \supset \\
 &\left\{\left | \frac{1}{n}\sum_{i=1}^n f(\<x_i,\th \>) - \E[f(\< x,\th\>)]  \right | > \e \right\}.
} 
By the union bound, we obtain
\eq{\label{eq::boundUnion1}
&\P \left(\max_{\th \in \T_\Delta} \left [ \frac{1}{n}\sum_{i=1}^n u_\th(x_i) - \E[u_\th(x)]  \right ] > \e/2 \right) +
\P \left( \max_{\th \in \T_\Delta}\left [- \frac{1}{n}\sum_{i=1}^n l_\th(x_i) + \E[l_\th(x)]  \right ] > \e/2\right)
\geq\nonumber \\ 
&\P \left( \sup_{\th \in B_p(R)} \left | \frac{1}{n}\sum_{i=1}^n f(\<x_i,\th \>) - \E[f(\< x,\th\>)]  \right | > \e \right).
}
In order to complete the proof, we need concentration inequalities for $u_\th$ and $l_\th$. We state the following lemma.

\begin{lemma}\label{lem::bracketCon}
There exists a constant $C$ depending on the eigenvalues of $\Sig$ and $B$ such that, for each $\th\in B_n(R)$ and for some $0<\epsilon <1$, we have
\eqn{
&\P \left( \left | \frac{1}{n}\sum_{i=1}^n u_\th(x_i) - \E[u_\th(x)]  \right | > \epsilon/2 \right)\leq
2e^{-C n\epsilon^2}, \\
&\P \left( \left | \frac{1}{n}\sum_{i=1}^n l_\th(x_i) - \E[l_\th(x)]  \right | > \epsilon/2\right)\leq
2e^{-C n\epsilon^2},
}
where 
\eqn{
C = \frac{c}{\left ( B+ \frac{\sqrt{2}K}{4\mu/\sqrt{p}}\right)^2}
}
for an absolute constant $c$. 
\end{lemma}
\begin{remark}
Note that $\mu =\E[\| x \|_2]  = \O(\sqrt{p})$ and dividing it by $\sqrt{p}$ would give a constant independent of $n$ and $p$.
\end{remark}
\begin{proof}[Proof of Lemma \ref{lem::bracketCon}]
By the relation between sub-gaussian and sub-exponential norms, we have
\eq{\label{eq::subGsubE}
 \| \|x \|_2 \|^2_\sg \leq  \| \|x \|^2_2 \|_\se \leq& \sum_{i=1}^p \|x_i^2\|_\se, \\ \nonumber
 \leq &2\sum_{i=1}^p \|x_i\|^2_\sg, \\\nonumber
 \leq& 2K^2 p.
}
Therefore $\| x \|_2 - \E[\| x \|_2]$ is a centered sub-gaussian random variable with sub-gaussian norm bounded above by $2K\sqrt{2p}$. We have,
\eqn{
\E[\| x \|_2]  = \mu.
}
Note that $\mu$ is actually of order $\sqrt{p}$. Assuming that the left hand side of the above equality is equal to $\sqrt{p}K'$ for some constant $K'>0$, we can conclude that
the random variable $u_\th(x)=f(\< x,\th\>) + \epsilon\frac{\| x\|_2}{4\E\left[\| x\|_2\right]}$ is also sub-gaussian with
\eqn{
\|u_\th(x)\|_\sg \leq& B +  \frac{\epsilon}{4\E\left[\| x\|_2\right]} \| \|x\|_2\|_\sg \\
& \leq B +  \frac{\epsilon}{4\sqrt{p}K'} K\sqrt{2p} \\
& \leq B +  C'
}
where $C'=\sqrt{2}K/4K'$ is a constant and we also assumed $\epsilon <1$.
Now, define the function
\eqn{
g_\th(x) = u_\th(x) - \E[u_\th(x)].
}
Note that $g_\th(x)$ is a centered sub-gaussian random variable with sub-gaussian norm
\eqn{
\|g_\th(x)\|_\sg \leq 2B + 2C'.
}
Then, by the Hoeffding-type inequality for the sub-gaussian random variables, we obtain
\eqn{
\P \left ( \left |\frac{1}{n}\sum_{i=1}^n g_\th(x_i)\right | > \epsilon/2\right)  \leq &
2e^{-c n\epsilon^2/(B+C')^2}
}
where $c$ is an absolute constant. 
The same argument follows for $l_\th(x)$.
\end{proof}

Using the above lemma with the union bound over the set $T_\Delta$, we can write
\eqn{
&\P \left( \sup_{\th \in B_n(R)} \left | \frac{1}{n}\sum_{i=1}^n f(\<x_i,\th \>) - \E[f(\< x,\th\>)]  \right | > \epsilon \right) \leq 4 |T_\Delta| e^{-C n\epsilon^2} .
}
Since we can also write, by Lemma \ref{lem::sphere}
\eqn{
|T_\Delta|\leq \left ( \frac{R\sqrt{p}}{\Delta}\right)^p
\leq& \left ( \frac{4RL\E[\|x\|_2]\sqrt{p}}{\epsilon}\right)^p,\\
\leq& \left ( \frac{4\sqrt{2}RLKp}{\epsilon}\right)^p,\\
}
and we observe that, for the constant $c'=4\sqrt{2}RLK$,
\eqn{
\P \left( \sup_{\th \in B_n(R)} \left | \frac{1}{n}\sum_{i=1}^n f(\<x_i,\th \>) - \E[f(\< x_1,\th\>)]  \right | > \epsilon \right)
&\leq 4\left ( \frac{4\sqrt{2}RLKp}{\epsilon}\right)^p e^{-C n\epsilon^2}, \\
&= 4 \exp\left\{p\log(c'p/\epsilon)-C n\epsilon^2\right\}.
}
We will obtain an exponential decay of order $p$ on the right hand side. For some constant $h$ depending on $n$ and $p$, if we choose $\epsilon=hp$, we need
\eqn{
h^2 \geq \frac{1}{Cnp} \log(c'/h).
}
By the Lemma \ref{lem::epsilon}, choosing $h^2 = \log(2c'^2Cnp)/(2Cnp)$, we satisfy the above requirement. Note that for $n$ large enough, the condition of the lemma is easily satisfied. Hence, for 
\eqn{
\epsilon^2 = \frac{ p\log(2c'^2Cnp)}{2Cn} = \O\left(\frac{p\log(n)}{n}\right), 
}
we obtain that there exists constants $c_1, c_2,c_3$ such that
\eqn{
\P \left( \sup_{\th \in B_n(R)} \left | \frac{1}{n}\sum_{i=1}^n f(\<x_i,\th \>) - \E[f(\< x_1,\th\>)]  \right | > c_1\sqrt{\frac{ p\log(n)}{n}} \right) \leq c_2e^{-c_3p},
}
where 
\eqn{
c_1 =& \frac{3\left ( B+ \frac{\sqrt{2}K}{4\sqrt{\trace(\Sig)/p - 16K^2}}\right)^2}{2c},\\
c_2 =& 4,\\
c_3 =&\frac{1}{2}\log(7) \leq \frac{1}{2} \log(\log(64R^2L^2K^2C) + 6\log(p)) .
}
when $p > e$ and $64R^2L^2K^2C>e$.

\end{proof}

In the following, we state the concentration results on the unbounded functions of the form 
$$x\to f(\<x,\th \>)\<x,v\>^2.$$ Functions of this type form the summands of the Hessian matrix in GLMs.
\begin{lemma}\label{lem::concentration4}
Let $x_i$, for $i=1,...,n$, be i.i.d sub-gaussian random variables with mean 0, covariance matrix $\Sig$ and sub-gaussian norm $K$.
Also let $f:\reals \to \reals$ be a uniformly bounded function such that for some $B >0 $, we have $\|f \|_\infty < B$ and $f$ is Lipschitz continuous with constant $L$. Further, let $v\in \reals^p$ such that $\|v\|_2=1$. Then, for $n,p$ sufficiently large satisfying $$n^{0.2}/\log(n)\gtrsim p$$, there exists absolute constants $c_1,c_2$ depending on $L,B$, $R$ and the eigenvalues of $\Sig$ such that, we have
\eqn{
\P \left( \sup_{\th \in B_p(R)} \left | \frac{1}{n}\sum_{i=1}^n f(\<x_i,\th \>)\<x_i,v\>^2 - \E[f(\< x,\th\>)\<x,v\>^2]  \right | > c_1\sqrt{\frac{p}{n^{0.2}} \log\left(n \right)} \right) \leq  c_2e^{-p} .\\
}
\end{lemma}
\begin{proof}[Proof of Lemma \ref{lem::concentration4}]
We define the brackets of the form
\eq{\label{eq::brackets}
l_\th(x) =& f(\< x,\th\>)\<x,v \>^2 - \e\frac{\| x\|^3_2}{4\E\left[\| x\|^3_2\right]},\nonumber\\
u_\th(x)=& f(\< x,\th\>)\<x,v \>^2 + \e\frac{\| x\|^3_2}{4\E\left[\| x\|^3_2\right]},
}
and we observe that the bracket $[\l_\th , u_\th]$ has size $\e/2$ in $L_1$, that is,
\eqn{
\E\left [|u_\th(x)-l_\th(x)| \right] = \e/2.
}

Next, for the following constant $$\Delta = \frac{\e}{4L \E\left[\| x\|^3_2\right]},$$ we define a $\Delta$-net over $B_p(R)$ and call it $\T_\Delta$. Then, $\forall \th \in B_p(R)$, $\exists \th'\in \T_\Delta$ such that $f(\< \cdot , \th \>)\<\cdot,v \>^2$ belongs to the bracket $[\l_{\th'} , u_{\th'}]$. This can be seen by writing the Lipschitz continuity of the function $f$, i.e.,
\eqn{
\left| f(\<x,\th\>)\<x,v \>^2-f(\<x,\th'\>)\<x,v \>^2\right| =& \<x,v \>^2\left| \left\{f(\<x,\th\>)-f(\<x,\th'\>)\right\}\right|, \\
 \leq& L \|x\|^2_2\  \|v\|_2^2 \   \left|\<x,\th-\th'\>\right|,\\
 \leq& L \|x\|^3_2 \|\th-\th'\|_2,\\
\leq& \Delta L \|x\|^3_2,
}
where we used Cauchy-Schwartz to obtain the above inequalities.
Hence, we may conclude that for the bracketing functions in \ref{eq::brackets}, bracketing number of the function class $$\F = \{  f(\<\cdot,\th\>)\<\cdot,v \>^2 : \th \in B_p(R)\}$$ is bounded above by the covering number of the ball of radius $R$ for the given scale $\Delta = \e /(4L \E\left[\| x\|^3_2\right])$, i.e.,
\eqn{
\N (\e/2, \F, L_1) \leq |\T_\Delta|.
}

Next, we will upper bound the target probability using the bracketing functions $u_\th, l_\th$. We have $\forall \th \in B_p(R)$, $\exists \th' \in \T_\Delta$ such that
\eqn{
 \frac{1}{n}\sum_{i=1}^n f(\<x_i,\th \>)\<x_i,v \>^2 - \E[f(\< x,\th\>)\<x,v \>^2] \leq& \  \frac{1}{n}\sum_{i=1}^n u_{\th'}(x_i) - \E[u_{\th'}(x)] +\frac{\e}{2}, \\
  \frac{1}{n}\sum_{i=1}^n f(\<x_i,\th \>)\<x_i,v \>^2 - \E[f(\< x,\th\>)\<x,v \>^2] \geq& \ \frac{1}{n}\sum_{i=1}^n l_{\th'}(x_i) - \E[l_{\th'}(x)] -\frac{\e}{2}.
}
Using the above inequalities, $\forall \th$, $\exists \th'$, we can write
\eqn{
 &\left\{\left [ \frac{1}{n}\sum_{i=1}^n u_{\th'}(x_i) - \E[u_{\th'}(x)]  \right ] > \e/2 \right\} \cup
 \left\{\left [ -\frac{1}{n}\sum_{i=1}^n l_{\th'}(x_i) + \E[l_{\th'}(x)]  \right ] > \e/2 \right\}  \supset \\
 &\left\{\left | \frac{1}{n}\sum_{i=1}^n f(\<x_i,\th \>)\<x_i,v \>^2 - \E[f(\< x,\th\>)\<x,v \>^2]  \right | > \e \right\}.
}
Hence, by the union bound, we obtain
\eq{\label{eq::boundUnion1}
&\P \left(\max_{\th \in \T_\Delta} \left [ \frac{1}{n}\sum_{i=1}^n u_\th(x_i) - \E[u_\th(x)]  \right ] > \e/2 \right) +
\P \left( \max_{\th \in \T_\Delta}\left [- \frac{1}{n}\sum_{i=1}^n l_\th(x_i) + \E[l_\th(x)]  \right ] > \e/2\right)
\geq\nonumber \\ 
&\P \left( \sup_{\th \in B_p(R)} \left | \frac{1}{n}\sum_{i=1}^n f(\<x_i,\th \>)\<x_i,v\>^2 - \E[f(\< x,\th\>)\<x,v\>^2]  \right | > \e \right).
}
In order to complete the proof, we need one-sided concentration inequalities for $u_\th$ and $l_\th$. Handling these functions is somewhat tedious since $\|x\|_2^3$ terms do not concentrate nicely. We state the following lemma.

\begin{lemma}\label{lem::bracketCon2}
For a given $\xx,\e >0$, and $n$ sufficiently large such that, $\nu(n^\xx,p,\e,B,K,\Sig) < \e/4$ where
\eqn{
\nu(n^\xx,p,\e,B,K,\Sig)
\eqqcolon & \left (n^\xx+\frac{6B^2K^2 p}{c}  \right) \exp \left(-c\frac{n^\xx}{6B^2K^2 p} \right)\\
&+\left \{n^\xx+  \frac{4K^2p}{c\trace{(\Sig)}}n^{\xx/3}\e^{2/3}
  + \frac{6K^4p^2}{c^2\trace(\Sig)^2 }\e^{4/3}n^{-\xx/3}
\right \}\exp\left( -c\frac{\trace(\Sig)(n^\xx /\e)^{2 /3}}{2K^2 p}\right)
.
}
Then, there exists constants $c',c'',c'''$ depending on the eigenvalues of $\Sig$, $B$ and $K$ such that $\forall \th$, we have,
\eqn{
&\P \left(  \frac{1}{n}\sum_{i=1}^n u_\th(x_i) - \E[u_\th(x)]   > \e/2 \right)\leq
\exp \left(-c'n^\xx /p \right)+
\exp\left( -c''n^{2\xx /3}\e^{-2/3}\right)
+\exp\left( -c''' n^{1-2\xx}\e^2\right) ,  \\
&\P \left(  -\frac{1}{n}\sum_{i=1}^n l_\th(x_i) + \E[l_\th(x)]   > \e/2\right)\leq
\exp \left(-c'n^\xx /p \right)+
\exp\left( -c''n^{2\xx /3}\e^{-2/3}\right)
+\exp\left( -c''' n^{1-2\xx}\e^2\right) .
}
\end{lemma}
\begin{proof}[Proof of Lemma \ref{lem::bracketCon2}]
For the sake of simplicity, we define the functions
\eqn{
\tu_\th (w) =& \ u_\th(w) - \E[u_\th(x)],\\
\tl_\th (w) =& \ l_\th(w) - \E[l_\th(x)].
}
We will show the concentration of the upper bracket, $\tu$. Proof for the lower bracket $\tl$, will follow from similar steps.
We write,
\eq{\label{eq::bracketBounds}
\P \left(  \frac{1}{n}\sum_{i=1}^n \tu_\th (x_i)   > \e/2\right) \leq &
\P \left( \frac{1}{n}\sum_{i=1}^n \tu_\th (x_i)  > \e/2 , \ \max_{1\leq i \leq n} |\tu_\th(x_i)| < n^{\xx}\right) \nonumber\\
&+ \P \left( \max_{1\leq i \leq n} |\tu_\th(x_i)| \geq n^{\xx}\right).
}

We need to bound the right hand side of the above equation. For the second term, since $\tu_\th(x_i)$'s are i.i.d. centered random variables,  we have
\eqn{
\P \left( \max_{1\leq i \leq n} |\tu_\th(x_i)| \geq n^{\xx}\right) =& 1-\P \left( \max_{1\leq i \leq n} |\tu_\th(x_i)| < n^{\xx}\right),\\
=& 1-\P \left(|\tu_\th(x)| < n^{\xx}\right)^n,\\
=& 1-\left(1-\P \left(|\tu_\th(x)| \geq n^{\xx}\right)\right)^n,\\
\leq& n\P \left(|\tu_\th(x)| \geq n^{\xx}\right).\\
}
Also, note that
\eqn{
|\tu_\th (x)| \leq& B\|x\|_2^2 + \e\frac{\| x\|^3_2}{4\E\left[\| x\|^3_2\right]} + \E[u_\th(x)], \\
\leq& B\|x\|_2^2 + \e\frac{\| x\|^3_2}{4\E\left[\| x\|^3_2\right]} + B\trace(\Sig) + \e/4.
}
Therefore, if $t>3B\trace(\Sig)$ and for $\epsilon$ small, we can write
\eq{\label{eq::absConstant1}
\left\{|\tu_\th(x)| > t \right\} \subset \left\{ B\|x\|_2^2 > t/3\right\} \cup
\left\{ \e\frac{\| x\|^3_2}{4\E\left[\| x\|^3_2\right]} > t/3\right\}.
}
Since $x$ is a sub-gaussian random variable with $\|x\|_\sg=K$, we have
\eqn{
K=\sup_{w\in S^{p-1}}\|\<w,x\>\|_\sg 
=  \|x\|_\sg .
}
Using this and the relation between sub-gaussian and sub-exponential norms as in Eq.~(\ref{eq::subGsubE}), we have
$
 \| \|x \|_2 \|^2_\sg 
 \leq 2K^2 p.
$
This provides the following tail bound for $\| x \|_2$,
\eq{\label{eq::absConstant}
\P \left (\| x \|_2 > s \right) \leq \exp\left( -\frac{cs^2}{2pK^2}\right),
}
where $c$ is an absolute constant. 
Using the above tail bound, we can write,
\eqn{
\P\left(\|x\|_2^2 > \frac{1}{3B}t\right) \leq 2 \exp \left(-c\frac{t}{6BK^2 p} \right).
}

For the next term in Eq. (\ref{eq::absConstant1}), we need a lower bound for $\E\left[\| x\|^3_2\right]$. We use a modified version of the H{\"o}lder's inequality and obtain
\eqn{
&\E\left[\|x\|^3_2\right] \geq \E\left[\|x\|^2_2\right]^{3/2}
=\trace(\Sig)^{3/2}.
}
Using the above inequality, we can write
\eqn{
\P\left( \e\frac{\| x\|^3_2}{4\E\left[\| x\|^3_2\right]} > t/3\right)
\leq&\P\left( \| x\|^3_2 > \frac{4}{3\e} \trace(\Sig)^{3/2}t\right),\\
=&\P\left( \| x\|_2 > \left(\frac{4t}{3\e}\right)^{1/3}\trace(\Sig)^{1/2}\right),\\
\leq& \exp\left( -c\frac{\trace(\Sig)(t/\e)^{2 /3}}{2K^2 p}\right).
}
where $c$ is the same absolute constant as in Eq. (\ref{eq::absConstant}). 

Now for $\xx >0 $ such that $t=n^\xx > 3B\trace(\Sig)$ (we will justify this assumption for a particular choice of $\xx$ later), 
we combine the above results,
\eq{\label{eq::boundUtilde}
\P\left(|\tu_\th(x)| > t \right) \leq 
2 \exp \left(-c\frac{t}{6BK^2 p} \right)
+2\exp\left( -c\frac{\trace(\Sig)(t/\e)^{2 /3}}{2K^2 p}\right).
}

Next, we focus on the first term in Eq.(\ref{eq::bracketBounds}). Let $\mu = \E[\tu_\th (x)\ind_{\{|\tu_\th(x)|<n^{\xx}\}}]$, and write
\eqn{
\P \left(  \frac{1}{n}\sum_{i=1}^n \tu_\th (x_i)  > \e/2 \ ; \ \max_{1\leq i \leq n} |\tu_\th(x_i)| < n^{\xx}\right)
\leq& \P \left(  \frac{1}{n}\sum_{i=1}^n \tu_\th (x_i)\ind_{\{|\tu_\th(x_i)|<n^{\xx}\}}   > \e/2 \right), \\
=& \P \left(  \frac{1}{n}\sum_{i=1}^n \tu_\th (x_i)\ind_{\{|\tu_\th(x_i)|<n^{\xx}\}} - \mu > \e/2 - \mu \right),\\
\leq& \exp\left\{ - \frac{n^{1-2\xx}}{2}\left(\frac{\e}{2}-\mu\right)^2\right\},
}
where we used the Hoeffding's concentration inequality for the bounded random variables. Further, note that
\eqn{
0=\E[\tu_\th(x)] = \mu + \E\left[ \tu_\th(x)\ind_{\{|\tu_\th(x)|>n^\xx\}}\right].
}
By Lemma \ref{lem::darthvader}, we can write
\eqn{
|\mu|= \left|\E\left[ \tu_\th(x)\ind_{\{|\tu_\th(x)|>n^\xx\}}\right]\right| \leq& 
n^\xx \P(|\tu_\th(x)|>n^\xx) + \int^\infty_{n^\xx}\P(|\tu_\th(x)|>t)dt.
}
The first term on the right hand side can be easily bounded by using Eq.(\ref{eq::boundUtilde}), i.e.,
\eqn{
n^\xx \P(|\tu_\th(x)|>n^\xx) \leq 
n^\xx \exp \left(-c\frac{n^\xx }{6BK^2 p} \right)
+n^\xx\exp\left( -c\frac{\trace(\Sig)(n^\xx /\e)^{2 /3}}{2K^2 p}\right).
}
For the second term, using Eq.(\ref{eq::boundUtilde}) once again, we obtain

\eqn{
\int^\infty_{n^\xx}\P(|\tu_\th(x)|>t)dt \leq &
\int^\infty_{n^\xx}  \exp \left(-c't /p \right)dt+\int^\infty_{n^\xx}\exp\left( - c\frac{\trace(\Sig)(t/\e)^{2 /3}}{4K^2 p}\right)dt,\\
=&\frac{6BK^2 p}{c}  \exp \left(-c\frac{n^\xx}{6BK^2 p} \right)
+\int^\infty_{n^\xx}\exp\left( -c\frac{\trace(\Sig)(t/\e)^{2 /3}}{2K^2 p}\right)dt.
}
Next, we apply Lemma \ref{lem::integralBound} to bound the second term on the right hand side. That is, we have
\eqn{
\int^\infty_{n^\xx}\exp\left( -c''(t/\e)^{2 /3}\right)dt 
\leq
\left \{ \frac{3K^2p}{c\trace{(\Sig)}}n^{\xx/3}\e^{2/3}
  + \frac{6K^4p^2}{c^2\trace(\Sig)^2 }\e^{4/3}n^{-\xx/3}
\right \}\exp\left( -c\frac{\trace(\Sig)(n^\xx /\e)^{2 /3}}{2K^2 p}\right).
}
Combining the above results, we can write
\eqn{
|\mu| \leq &
\left (n^\xx+\frac{6B^2K^2 p}{c}  \right) \exp \left(-c\frac{n^\xx}{6B^2K^2 p} \right)\\
&+\left \{n^\xx+  \frac{4K^2p}{c\trace{(\Sig)}}n^{\xx/3}\e^{2/3}
  + \frac{6K^4p^2}{c^2\trace(\Sig)^2 }\e^{4/3}n^{-\xx/3}
\right \}\exp\left( -c\frac{\trace(\Sig)(n^\xx /\e)^{2 /3}}{2K^2 p}\right),\\
\eqqcolon & \nu(n^\xx,p,\e,B,K,\Sig).
}
Notice that, the upper bound on $|\mu|$, namely $\nu(n^\xx,p,\e,B,K,\Sig)$, is close to 0 when $n$ is large. This is because exponentially decaying functions will dominate the coefficients. We assume that $n$ is sufficiently large that the upper bound for $|\mu|$ is less than $\e/4$. In particular, we will choose $\xx = 0.4$ later in the proof.

Applying this bounds on Eq.(\ref{eq::bracketBounds}), we obtain
\eqn{
\P \left(  \frac{1}{n}\sum_{i=1}^n \tu_\th (x_i)   > \e/2\right) \leq &
 \exp \left(-c\frac{n^\xx}{6B^2K^2 p} \right)
+\exp\left( -c\frac{\trace(\Sig)(n^\xx/\e)^{2 /3}}{2K^2 p}\right)
+\exp\left( - \frac{n^{1-2\xx}}{8}\e^2\right), \\
= & \exp \left(-c'n^\xx /p \right)+
\exp\left( -c''n^{2\xx /3}\e^{-2/3}\right)
+\exp\left( -c''' n^{1-2\xx}\e^2\right),
}
where
\eqn{
c'=& \frac{c}{6B^2K^2},\\
c''=& \frac{c\trace(\Sig)/p}{2K^2}\geq \frac{c\sigma^2}{2K^2},\\
c'''=&\frac{1}{8} .
}
Hence, the proof is completed for the upper bracket.

The proof for the lower brackets $l_\th(x)$ follows from exactly the same steps and omitted here. 
\end{proof}
Applying the above lemma on Eq.(\ref{eq::boundUnion1}), for $\xx>0$, we obtain
\eq{\label{eq::decayExp}
\P \left( \sup_{\th \in B_n(R)} \left | \frac{1}{n}\sum_{i=1}^n f(\<x_i,\th \>)\<x_i,v\>^2 - \E[f(\< x,\th\>)\<x,v\>^2]  \right | > \e \right)& 
\leq  2 |T_\Delta|\exp \left(-c'n^\xx /p \right)\nonumber\\
&+
2|T_\Delta|\exp\left( -c''n^{2\xx /3}\e^{-2/3}\right)\nonumber\\
&+2|T_\Delta|\exp\left( -c''' n^{1-2\xx}\e^2\right).
}
Observe that we can write, by Lemma \ref{lem::sphere}
\eqn{
|T_\Delta|\leq& \left ( \frac{R\sqrt{p}}{\Delta}\right)^p
= \left ( \frac{4\sqrt{p}RL\E[\|x\|_2^3]}{\e}\right)^p.
}
Also, recall that $\|x\|_2$ was a sub-gaussian random variable with $\|\|x\|_2\|_\sg \leq K\sqrt{2p}$. Using the definition of sub-gaussian norm, we have
\eqn{
\frac{1}{\sqrt{3}}\,\E[\| x \|_2^3]^{1/3} \leq &\|\|x\|_2\|_\sg  \leq \sqrt{2p}\,K,\\
\implies& \E[\| x \|_2^3] \leq 15K^3p^{3/2}.
}
Therefore,  we have $\E[\| x \|_2^3] = \O (p^{3/2})$ (recall that we had a lower bound of the same order). We define a constant $K'$, and as $\epsilon$ is small, we have
\eqn{
|T_\Delta|\leq
& \left ( \frac{60RLK^3 p^{2}}{\e}\right)^p\\
 =&\left ( \frac{K' p^{2}}{\e}\right)^p,
}
where we let $K' = 60RLK^3$. We will show that each term on the right hand side of 
Eq.(\ref{eq::decayExp}) decays exponentially with a rate of order $p$. For the first term, for $s>0$, we write

\eq{\label{eq::expDecay1}
 |T_\Delta|\exp \left(-c'n^\xx /p \right) = &\exp \left(-c'n^\xx /p + p\log(K') + 2 p\log(p) + p\log(\e^{-1}) \right),
 \nonumber
 \\
  \leq &\exp \left(-c'n^\xx /p +  2 p\log(K'p/\e) \right).
}

Similarly for the second and third terms, we write
\eq{
|T_\Delta|\exp\left( -c''n^{2\xx /3}\e^{-2/3}\right) \leq& 
\exp\left( -c''n^{2\xx /3}\e^{-2/3}+ 2 p\log(K'p/\e) \right),\label{eq::expDecay2}\\
|T_\Delta|\exp\left( -c''' n^{1-2\xx}\e^2\right) \leq &
\exp\left( -c''' n^{1-2\xx}\e^2+ 2 p\log(K'p/\e) \right).
\nonumber
}
We will seek values for $\e$ and $\xx$ to obtain an exponential decay with rate $p$ on the right sides of Eq.(\ref{eq::expDecay1},\ref{eq::expDecay2}). That is, we need
\eq{\label{eq::expDecay3}
c'n^\xx /p \geq &2 p\log(K''p/\e),\\
c''n^{2\xx /3} \geq &2 p\log(K''p/\e)\e^{2/3},\nonumber\\
c''' n^{1-2\xx}\e^2 \geq &2 p\log(K''p/\e),\nonumber
}
where $K'' = e K'$.


We apply Lemma \ref{lem::epsilon} for the last inequality in Eq. (\ref{eq::expDecay3}). That is,
\eqn{
\e^2 =& \frac{p}{c'''n^{1-2\alpha}} \log\left(c'''{K''} n^{1-2\alpha} \right),\\
=&\O\left( \frac{p}{n^{1-2\alpha}} \log\left(n \right)  \right).
}
where we assume that 
\eqn{
n^{1-2\alpha}p > \frac{e}{c''' {K''}^2}.
}
The above statement holds for $\alpha<1/2$ and $n$ large enough, as the right hand side is an absolute constant. 

In the following, we choose $\alpha = 0.4$ and use the assumption that
\eqn{
n^{0.2}/\log(n)\gtrsim p,
} 
which provides $\e < 1$. Note that this choice of $\xx$ also justifies the assumption used to derive Eq.~(\ref{eq::boundUtilde}).
One can easily check that $\xx=0.4$ implies that the first and the second statements in Eq. (\ref{eq::expDecay3}) are satisfied.

It remains to check whether $\nu(n^\xx,p,\e,B,K,\Sig) < \e/4$ (in Lemma \ref{lem::bracketCon2}) for this particular choice of $\xx$ and $\e$. It suffices to consider only the dominant terms in the definition of $\nu$. 
We use the assumption on $n,p$ and write
\eqn{
\nu(n^{0.2},p,\e,B,K,\Sig) \lesssim&
n^{0.2} \exp\left (-\frac{cn^{0.4}}{6B^2K^2p} \right)+
n^{0.2} \exp\left (-\frac{c\trace(\Sig)/p\ n^{0.8/3}}{2K^2} \right),\\
\lesssim&
n^{0.2} \exp\left (-\frac{c}{6B^2K^2}n^{0.2} \right)+
n^{0.2} \exp\left (-\frac{c\sigma^2}{2K^2}n^{0.8/3} \right).\\
}
For $n$ sufficiently large, this quantity is always smaller than $\e/4$.
Hence, for some constant $c_1,c_2$, we obtain
\eqn{
\P \left( \sup_{\th \in B_p(R)} \left | \frac{1}{n}\sum_{i=1}^n f(\<x_i,\th \>)\<x_i,v\>^2 - \E[f(\< x,\th\>)\<x,v\>^2]  \right | > c_1\sqrt{\frac{p}{n^{0.2}} \log\left(n \right) }\right) \leq  c_2e^{-p} ,
}
where 
\eqn{
c_1 &= 16,\\
c_2 &= 6.
}
\end{proof}

\section{Proof of main theorem}
We will provide the proofs of Theorems \ref{thm::mainBounded} and \ref{thm::mainSG} in parallel. Matrix concentration results in this section are mostly based on the covering net argument provided in \cite{vershynin2010introduction}. 
Similar results for matrix forms can also be obtained through different techniques such as \emph{chaining} as well (See i.e. \cite{dicker2015flexible}).
On the set $\mathcal{E}$, we write,
\eqn{
 \hth^{t} -\th_*- \gamma \Q^t \grad_\th\l(\hth^{t})
=& \ \hth^{t} -\th_*- \gamma \Q^t \int_0^1\gradS_\th \l(\th_* + \xi(\hth^{t}-\th_*))d\xi (\hth^{t} -\th_*), \\
=& \ \left( I - \gamma \Q^t \int_0^1\gradS_\th \l(\th_* + \xi(\hth^{t}-\th_*))d\xi\right) (\hth^{t} -\th_*)\,.
}
Since the projection $\proj_{B_p(R)}$ in step 3 of \ALG can only decrease the $\ell_2$ distance, we obtain
\eq{\label{eq::ineq-govern}
\| \hth^{t+1} -\th_* \|_2 \leq&
 \left\| I - \gamma \Q^t \int_0^1\gradS_\th \l (\th_* + \xi(\hth^{t}-\th_*))d\xi\right\|_2 \|\hth^{t} -\th_*\|_2 .
}
%

The governing term (with $\gamma=1$) that determines the convergence rate can be bounded as
\eqn{
\left\| I -  \Q^t \int_0^1\gradS_\th \l (\th_* + \xi(\hth^{t}-\th_*))d\xi\right\|_2\leq
 \left\| [\Q^t]^{-1} -   \int_0^1\gradS_\th \l (\th_* + \xi(\hth^{t}-\th_*))d\xi\right\|_2 \| \Q^t \|_2.
}
We define the following,
\eqn{
\EE (\th) = \E\left[\ddphi(\<x,\th\>)\right]\Sig + 
\E\left[\ddddphi(\<x,\th\>)\right]\Sig\th\th^T\Sig 
}
Note that for a function $f$, $\E[f(\<x,\th\>)] = h(\th)$ is a function of $\th$. With a slight abuse of notation, we write $\E[f(\<x,\hth\>)] = h(\hth)$ as a random variable. We have
\eqn{
 \Big\| [\Q^t]^{-1} -   \int_0^1\gradS_\th \l &(\th_* + \xi(\hth^{t}-\th_*))d\xi   \Big\|_2 \leq
\left\|[\Q^t]^{-1} - \EE (\hth^t)\right\|_2  \\
+& \left\|[\E [ xx^T\ddphi(\<x,\hth^t\>)] - \EE (\hth^t)\right\|_2  \\
 +&\left\|\int_0^1\gradS_\th \l (\th_* + \xi(\hth^{t}-\th_*))d\xi - \E \left[ xx^T\int_0^1 \ddphi (\<x,\th_* + \xi(\hth^{t}-\th_*)\>)d\xi\right]\right\|_2 \\
 + & \left\|\E [ xx^T\ddphi(\<x,\hth^t\>)] - \E \left[ xx^T\int_0^1 \ddphi (\<x,\th_* + \xi(\hth^{t}-\th_*)\>)d\xi\right]\right\|_2\, .
}

For the first term on the right hand side, we state the following lemma.
%
%
%

\begin{lemma}\label{lem::term1}
In the case of sub-gaussian covariates, there exist constants $C_1,C_2$ such that, with probability at least $1-C_1/p^2$,

\eqn{
\left\|[\Q^t]^{-1} - \EE (\hth^t)\right\|_2
\leq C_2 \sqrt{\frac{p}{\min\left\{ |S|p/\log(p),n/\log(n)\right\}}}.\\
}
Similarly, when the covariates are sampled from a distribution with bounded support, there exists constants $C_1',C_2',C_3'$ such that, with probability $1-C_1'e^{-C_2'p}$,
\eqn{
\left\|[\Q^t]^{-1} - \EE (\hth^t)\right\|_2
\leq C_3' \sqrt{\frac{p}{\min\left\{ |S|,n/\log(n)\right\}}},\\
}
where the constants depend on $K$, $B$ and the radius $R$. 
\end{lemma}
\begin{proof}[Proof of Lemma \ref{lem::term1}]
In the following, we will only provide the proof for the bounded support case. The proof for the sub-gaussian covariates can be carried by replacing Lemmas \ref{lem::concentration1bounded} and \ref{lem::concentration2bounded} with Lemmas \ref{lem::concentration1} and \ref{lem::concentration2}. 

Using a uniform bound on the feasible set, we write
\eqn{
&\left\|[\Q^t]^{-1} - \EE (\hth^t)\right\|_2
\\& \leq 
\sup_{\th \in B_p(R)} \left\|\hmu_2(\th)\zeta_r(\SigH_S) + \hmu_4(\th)\zeta_r(\SigH_S)\th(\zeta_r(\SigH_S)\th)^T 
-  
\E[\ddphi(\<x,\th\>)]\Sig - \E[\ddddphi(\<x,\th\>)]\Sig\th\th^T\Sig \right\|_2.
}

We will find an upper bound for the quantity inside the supremum. By denoting the expectations of $\hmu_2(\th)$ and $\hmu_4(\th)$, with $\mu_2(\th)$ and $\mu_4(\th)$ respectively,
we write
\eqn{
&\left\|\hmu_2(\th)\zeta_r(\SigH_S) + \hmu_4(\th)\zeta_r(\SigH_S)\th(\zeta_r(\SigH_S)\th)^T -  \E[\ddphi(\<x,\th\>)]\Sig - \E[\ddddphi(\<x,\th\>)]\Sig\th(\Sig\th)^T \right\|_2 \\
&\leq \left \|\hmu_2(\th)\zeta_r(\SigH_S)  - \mu_2(\th)\Sig\right\|_2+ \left\|\hmu_4(\th)\zeta_r(\SigH_S)\th(\zeta_r(\SigH_S)\th)^T -  \mu_4(\th)\Sig\th(\Sig\th)^T\right\|_2.
}
For the first term on the right hand side,  we have
\eqn{
\left\|\hmu_2(\th)\zeta_r(\SigH_S)  - \mu_2(\th)\Sig\right\|_2 \leq 
&|\hmu_2(\th)|\left\|\zeta_r(\SigH_S) - \Sig \right\|_2 + \|\Sig\|_2 \left|\hmu_2(\th) -\mu_2(\th)\right|,\\
\leq &B_2\left\|\zeta_r(\SigH_S) - \Sig \right\|_2 + K \left|\hmu_2(\th) -\mu_2(\th)\right|.\\
}
By the Lemmas \ref{lem::concentration1bounded} and \ref{lem::concentration2bounded}, for some constants $c_1,c_2,c_3$, we have with probability $1-c_2e^{-c_3p}-1/p^2$,
\eqn{
\sup_{\{\th\in B_p(R)\}}\left\|\hmu_2(\th)\zeta_r(\SigH_S)  - \mu_2(\th)\Sig\right\|_2 \leq &
2B_2K\sqrt{C}\sqrt{\frac{\log(p)}{|S|}} + c_1K \sqrt{\frac{p\log(n)}{n}},\\
=& \O\left ( \sqrt{\frac{p}{\min\left\{ p/\log(p)|S|,n/\log(n)\right\}}}\right).
}

For the second term, we have

\eqn{
&\left\|\hmu_4(\th)\zeta_r(\SigH_S)\th(\zeta_r(\SigH_S)\th)^T -  \mu_4(\th)\Sig\th(\Sig\th)^T\right\|_2 \\
&\leq |\hmu_4(\th)|\left\|\zeta_r(\SigH_S)\th\th^T\zeta_r(\SigH_S) - \Sig\th\th^T\Sig\right\|_2 + |\hmu_4(\th) - \mu_4(\th)|\left \|\Sig\th\th^T\Sig \right\|_2, \\
&\leq B_4R^2\left\{ \|\zeta_r(\SigH_S) \|_2 + \| \Sig\|_2\right\} \left\|\zeta_r(\SigH_S) - \Sig \right\|_2 + 
R^2 \| \Sig\|_2^2|\hmu_4(\th) - \mu_4(\th)|,  \\
&\leq B_4R^2\left\{ \|\zeta_r(\SigH_S) \|_2 + K\right\} \left\|\zeta_r(\SigH_S) - \Sig \right\|_2 + 
R^2 K^2|\hmu_4(\th) - \mu_4(\th)|.  \\
}

Again, by the Lemmas \ref{lem::concentration1bounded}, \ref{lem::concentration2bounded} and \ref{lem::concentrationBounded2}, for some constants $c_1,c_2,c_3$, we have with probability $1-c_2e^{-c_3p}-1/p^2$, we write
\eqn{
B_4R^2\left\{ \|\zeta_r(\SigH_S) \|_2 + K\right\} \left\|\zeta_r(\SigH_S) - \Sig \right\|_2 \leq&
2K\sqrt{C}B_4R^2\left \{  2K + 3K\sqrt{C}\sqrt{\frac{\log(p)}{|S|}}\right\}\sqrt{\frac{\log(p)}{|S|}},\\
\leq&
4K^2\sqrt{C}B_4R^2\sqrt{\frac{\log(p)}{|S|}} + 6K^2CB_4R^2\frac{\log(p)}{|S|},\\
=& \O\left(\sqrt{\frac{\log(p)}{|S|}}  \right),
}
for sufficiently large $|S|$.

Further, by Lemma \ref{lem::concentrationBounded1}, for constants $c_1,c_2,c_3$, we have with probability $1-c_2e^{-c_3p}$, 
\eqn{
 \sup_{\{\th\in B_p(R)\}}
 |\hmu_4(\th) - \mu_4(\th)| \leq c_1 \sqrt{\frac{p\log(n)}{n}} = \O\left(\sqrt{\frac{p\log(n)}{n}} \right).
}
Combining the above results, for sufficiently large $p,|S|$ and constants $c_1,c_2$, we have with probability at least $1-c_1/p^2$, 
\eqn{
\sup_{\{\th\in B_p(R)\}}
&\left\|\hmu_4(\th)\zeta_r(\SigH_S)\th(\zeta_r(\SigH_S)\th)^T -  \mu_4(\th)\Sig\th(\Sig\th)^T\right\|_2 \\
&\leq 
4K^2\sqrt{C}\max\{B_2,B_4\}R^2\sqrt{\frac{\log(p)}{|S|}} + 6K^2CB_4R^2\frac{\log(p)}{|S|}
+c_1R^2 K^2 \sqrt{\frac{p\log(n)}{n}}\\
&= \O\left ( \sqrt{\frac{p}{\min\left\{ |S|p/\log(p),n/\log(n)\right\}}}\right).
}

Hence, for some constants $C_1,C_2$, with probability $1-C_1/p^2$, we have
\eqn{
\left\|[\Q^t]^{-1} - \EE (\hth^t)\right\|_2 \leq 
C_2 \sqrt{\frac{p}{\min\left\{ |S|p/\log(p),n/\log(n)\right\}}},
}
where the constants depend on $K,B = \max\{B_2,B_4\}$ and the radius $R$.
\end{proof}


\begin{lemma}\label{lem::term2}
The bias term can be upper bounded by
\eqn{
\left\|\E [ xx^T\ddphi(\<x,\hth^t\>)] - \EE (\hth^t) \right\|_2 \leq
d_{\Hh_3}(x,z)+
\|\Sig\|_2 \ d_{\Hh_1}(x,z) + \|\Sig\|_2^2R^2\ d_{\Hh_2} (x,z),
}
for both sub-gaussian and bounded support cases.
\end{lemma}
\begin{proof}[Proof of Lemma \ref{lem::term2}]
For a random variable $z\sim \normal_p(0,\Sig)$, by the triangle inequality, we write
\eqn{
&\left\|\E [ xx^T\ddphi(\<x,\hth^t\>)] - \EE (\hth^t)\right\|_2 \\ &\leq
\left\|\E [ xx^T\ddphi(\<x,\hth^t\>)] - \E [ zz^T\ddphi(\<z,\hth^t\>)]\right\|_2
+\left\|\E [ zz^T\ddphi(\<z,\hth^t\>)] - \EE (\hth^t)\right\|_2 %
}
For the first term on the right hand side, we have
\eqn{
&\left\|\E [ xx^T\ddphi(\<x,\hth^t\>)] - \E [ zz^T\ddphi(\<z,\hth^t\>)]\right\|_2\\
&\leq \sup_{\th\in B_p(R)}\sup_{\|v\|_2=1} \left | \E\left [\<v,x\>^2 \ddphi(\<x,\th\>) \right]
-\E\left [\<v,z\>^2 \ddphi(\<z,\th\>) \right]\right|,\\
&\leq d_{\Hh_3}(x,z).
}
For the second term, we write
\eqn{
&\left\|[\E [ zz^T\ddphi(\<z,\hth^t\>)] - \EE (\hth^t)\right\|_2 \\
&\leq\sup_{\{\th\in B_p(R)\}}
\Big\|\E [ zz^T\ddphi(\<z,\th\>)] - \E[\ddphi(\<x,\th\>)]\Sig 
+\E\left[\ddddphi(\<x,\th\>)\right]\Sig\th\th^T\Sig\Big\|_2, \\ 
& \leq 
\sup_{\{\th\in B_p(R)\}}
\Big\|\E[\ddphi(\<z,\th\>)]\Sig +
\E\left[\ddddphi(\<z,\th\>)\right]\Sig\th\th^T\Sig \\
&\ \ \ \ \  \ \ \ \ \ \ \ - \E[\ddphi(\<x,\th\>)]\Sig -
\E\left[\ddddphi(\<x,\th\>)\right]\Sig\th\th^T\Sig\Big\|_2, \\
& \leq 
\sup_{\{\th\in B_p(R)\}}
\left\|\E[\ddphi(\<z,\th\>)]\Sig - \E[\ddphi(\<x,\th\>)]\Sig \right\|_2,\\
&\ \ \ \ \ \ \ \ \ \ + \sup_{\{\th\in B_p(R)\}}
\left\|
\E\left[\ddddphi(\<z,\th\>)\right]\Sig\th\th^T\Sig  -
\E\left[\ddddphi(\<x,\th\>)\right]\Sig\th\th^T\Sig\right\|_2, \\
& \leq \|\Sig\|_2
\sup_{\{\th\in B_p(R)\}}
\left|\E[\ddphi(\<z,\th\>)] - \E[\ddphi(\<x,\th\>)] \right| \\ 
&\ \ \ \ \  \ \ \ \ \ \ \ +\|\Sig\|_2^2R^2 \sup_{\{\th\in B_p(R)\}}
\left|\E[\ddddphi(\<z,\th\>)] - \E[\ddddphi(\<x,\th\>)] \right|,\\
& \leq \|\Sig\|_2 d_{\Hh_1}(x,z) + \|\Sig\|_2^2R^2 d_{\Hh_2} (x,z).
}

Hence, we conclude that
\eqn{
&\left\|\E [ xx^T\ddphi(\<x,\hth^t\>)] - \EE (\hth^t)\right\|_2 \leq
d_{\Hh_3}(x,z)+
\|\Sig\|_2 \ d_{\Hh_1}(x,z) + \|\Sig\|_2^2R^2\ d_{\Hh_2} (x,z).
}
\end{proof}


\begin{lemma}\label{lem::term3}
There exists constants $c_1,c_2,c_3$ depending on the eigenvalues of $\Sig$, $B,L$ and $R$ such that, with probability at least $1-c_2e^{-c_3p}$

\eqn{
\left\|\frac{1}{n}\sum_{i=1}^n  x_ix_i^T\int_0^1 \ddphi (\<x_i,\th_* + \xi(\hth^{t}-\th_*)\>)d\xi - \E \left[ xx^T\int_0^1 \ddphi (\<x,\th_* + \xi(\hth^{t}-\th_*)\>)d\xi\right]\right\|_2 \leq \delta,
}
where $\delta =  c_1\sqrt{\frac{p}{n^{0.2}} \log\left(n \right)}$ for sub-gaussian covariates, and 
$\delta =  c_1\sqrt{\frac{p}{n} \log\left(n \right)}$ for covariates with bounded support.
\end{lemma}
\begin{proof}
We provide the proof for bounded support case. The proof for sub-gaussian case can be carried by replacing Lemma \ref{lem::concentrationBounded2} with Lemma \ref{lem::concentration4}.

By the Fubini's theorem, we have

\eqn{
&\left\|\frac{1}{n}\sum_{i=1}^n  x_ix_i^T\int_0^1 \ddphi (\<x_i,\th_* + \xi(\hth^{t}-\th_*)\>)d\xi - \E \left[ xx^T\int_0^1 \ddphi (\<x,\th_* + \xi(\hth^{t}-\th_*)\>)d\xi\right]\right\|_2 ,\\
&=\left\|\int_0^1\left\{\frac{1}{n}\sum_{i=1}^n  x_ix_i^T \ddphi (\<x_i,\th_* + \xi(\hth^{t}-\th_*)\>) - \E \left[ xx^T \ddphi (\<x,\th_* + \xi(\hth^{t}-\th_*)\>)\right]\right\}d\xi\right\|_2, \\
&\leq\int_0^1\left\|\left\{\frac{1}{n}\sum_{i=1}^n  x_ix_i^T \ddphi (\<x_i,\th_* + \xi(\hth^{t}-\th_*)\>) - \E \left[ xx^T \ddphi (\<x,\th_* + \xi(\hth^{t}-\th_*)\>)\right]\right\}\right\|_2 d\xi,\\
&\leq\sup_{\th \in B_p(R)}\left\|\frac{1}{n}\sum_{i=1}^n  x_ix_i^T \ddphi (\<x_i,\th\>) - \E \left[ xx^T \ddphi (\<x,\th\>)\right]\right\|_2. \\
}

Using the definition of \emph{operator norm}, the right hand side is equal to

\eqn{
&\sup_{\th\in B_p(R)}\left\|\frac{1}{n}\sum_{i=1}^n  x_ix_i^T \ddphi (\<x_i,\th\>) - \E \left[ xx^T \ddphi (\<x,\th\>)\right]\right\|_2 \\
&=\sup_{\th\in B_p(R)}\sup_{v\in S^{p-1}}\left|\frac{1}{n}\sum_{i=1}^n  \ddphi (\<x_i,\th\>)\< x_i,v\>^2 - \E \left[  \ddphi (\<x,\th\>)\< x,v\>^2\right]\right|,
}

where $S^{p-1}$ denotes the $p$-dimensional unit sphere.

For $\Delta=0.25$, let $T_{\Delta}$ be an $\Delta$-net over $S^{p-1}$. Using Lemma \ref{lem::epsNet}, we obtain

\eqn{
&\P\left(\sup_{\th\in B_p(R)}\sup_{v\in S^{p-1}}\left|\frac{1}{n}\sum_{i=1}^n  \ddphi (\<x_i,\th\>)\< x_i,v\>^2 - \E \left[  \ddphi (\<x,\th\>)\< x,v\>^2\right]\right| >\epsilon\right),\\
&\leq \P\left( \sup_{\th\in B_p(R)}\sup_{v\in T_\Delta}\left|\frac{1}{n}\sum_{i=1}^n  \ddphi (\<x_i,\th\>)\< x_i,v\>^2 - \E \left[  \ddphi (\<x,\th\>)\< x,v\>^2\right]\right| >\epsilon/2\right),\\
&\leq |T_\Delta|\P\left( \sup_{\th\in B_p(R)}\left|\frac{1}{n}\sum_{i=1}^n  \ddphi (\<x_i,\th\>)\< x_i,v\>^2 - \E \left[  \ddphi (\<x,\th\>)\< x,v\>^2\right]\right| >\epsilon/2\right),\\
&= 9^p\P\left( \sup_{\th\in B_p(R)}\left|\frac{1}{n}\sum_{i=1}^n  \ddphi (\<x_i,\th\>)\< x_i,v\>^2 - \E \left[  \ddphi (\<x,\th\>)\< x,v\>^2\right]\right| >\epsilon/2\right).
}

By applying Lemma \ref{lem::concentrationBounded2} to the last line above, there exists absolute constants $c_1',c_2',c_3'$ depending on $L,B, R,K$ such that, we have

\eqn{
\P \left( \sup_{\th \in B_p(R)} \left | \frac{1}{n}\sum_{i=1}^n \ddphi(\<x_i,\th \>)\<x_i,v\>^2 - \E[\ddphi(\< x,\th\>)\<x,v\>^2]  \right | > c_1'\sqrt{\frac{p}{n} \log\left(n \right)} \right) \leq  c_2'e^{-c_3'p} .\\
}

$c_3'$ is of order $\O(p\log\log (n))$. Therefore, by choosing $n$ large enough, we 
obtain that there exists constants $c_1,c_2,c_3$ such that with probability at least $1-c_2e^{-c_3p}$

\eqn{
\sup_{\th\in B}\left\|\frac{1}{n}\sum_{i=1}^n  x_ix_i^T \ddphi (\<x_i,\th\>) - \E \left[ xx^T \ddphi (\<x,\th\>)\right]\right\|_2 \leq 
c_1\sqrt{\frac{p}{n} \log\left(n \right)}
}
\end{proof}


\begin{lemma}\label{lem::term4}
There exists a constant $C$ depending on $K$ and $L$ such that,

$$\left\|\E [ xx^T\ddphi(\<x,\hth^t\>)] - \E \left[ xx^T\int_0^1 \ddphi (\<x,\th_* + \xi(\hth^{t}-\th_*)\>)d\xi\right]\right\|_2 \leq \tilde{C} \|\hth^t - \th_*\|_2 ,$$
where $\tilde{C} = C$ for the bounded support case and $\tilde{C} = Cp^{1.5}$ for the sub-gaussian case.
\end{lemma}
\begin{proof}
By the Fubini's theorem, we write

\eqn{
&\left\|\E [ xx^T\ddphi(\<x,\hth^t\>)] - \E \left[ xx^T\int_0^1 \ddphi (\<x,\th_* + \xi(\hth^{t}-\th_*)\>)d\xi\right]\right\|_2,\\
&=\left\|\int_0^1\E \left[ xx^T\left \{\ddphi(\<x,\hth^t\>) -   \ddphi (\<x,\th_* + \xi(\hth^{t}-\th_*)\>)\right\}\right]d\xi\right\|_2,
}
Moving the integration out, right hand side of above equation is smaller than
\eqn{
&\int_0^1\left\|\E \left[ xx^T\left \{\ddphi(\<x,\hth^t\>) -   \ddphi (\<x,\th_* + \xi(\hth^{t}-\th_*)\>)\right\}\right]\right\|_2d\xi,\\
&\leq\int_0^1\left\|\E \left[ xx^T L| \< x,(1-\xi)(\hth^t - \th_*)\>| \right]\right\|_2d\xi,\\
&\leq  \E \left[  \|x\|^3_2 \|\hth^t - \th_*\|_2 \right]L\int_0^1(1-\xi)d\xi, \\
&=\frac{L\E[\|x\|_2^3]}{2}   \|\hth^t - \th_*\|_2 .
}

We observe that, when the covariates are supported in the ball of radius $\sqrt{K}$, we have $\E[\|x\|_2^3] \leq K^{3/2}$. When they are sub-gaussian random variables with norm $K$, we have 
$\E[\|x\|_2^3] \leq K^{3}6^{1.5}p^{1.5}$.

\end{proof}

By combining above results, for sub-gaussian covariates we obtain
\eqn{
& \Big\| [\Q^t]^{-1} -   \int_0^1\gradS_\th \l (\th_* + \xi(\hth^{t}-\th_*))d\xi   \Big\|_2 \\
& \ \ \ \ \leq
\D(x,z)
 +c_1  \sqrt{\frac{p}{\min\left\{ |S|p/\log(p),n/\log(n)\right\}}}
+c_2 \|\hth^t - \th_*\|_2\, ,
}
and for covariates with bounded support, we obtain
\eqn{
& \Big\| [\Q^t]^{-1} -   \int_0^1\gradS_\th \l (\th_* + \xi(\hth^{t}-\th_*))d\xi   \Big\|_2 \\
& \ \ \ \ \leq
\D(x,z)
 +c_1  \sqrt{\frac{p}{\min\left\{ |S|,n/\log(n)\right\}}}
+c_2p^{1.5} \|\hth^t - \th_*\|_2\, ,
}
where
\eqn{
\D(x,z)=d_{\Hh_3}(x,z)+
\|\Sig\|_2 \ d_{\Hh_1}(x,z) + \|\Sig\|_2^2R^2\ d_{\Hh_2} (x,z)\, .
}

In the following, we will derive an upper bound for $\left\| \Q^t \right\|_2$ where,
\eqn{
\Q^t = \frac{1}{\hmu_2(\hth^t)}\left[\zeta_r(\SigH_S)^{-1} - \frac{\hth^t[\hth^t]^T}{\hmu_2(\hth^t)/\hmu_4(\hth^t) + \<\zeta_r(\SigH_S)\hth^t, \hth^t\>}\right].
}
We define
\[
c_L = \inf_{\beta\in B_p(L)} \mu_2(\beta).
\]
Thus, for any iterate $\hth^t$ of Newton-Stein algorithm
\eqn{
\mu_2(\hth^t) \geq c_R.\\
}
By Lemmas \ref{lem::concentrationBounded1} and \ref{lem::concentration3}, for some constants $c_1,c_2,c_3$, with probability $1-c_2e^{-c_3p}$,
\eqn{
\hmu_2(\hth^t) \geq& \mu_2(\hth^t)-c_1 \sqrt{\frac{p \log(n)}{n}},\\
\geq &c_R-c_1 \sqrt{\frac{p \log(n)}{n}}.\\
}

Also, by the assumption given in the theorem, on the set $\mathcal{E}$ we have almost surely,
\eqn{
\inf_{t\geq 0}\left | \mu_2(\hth^t) + \mu_4(\hth^t)\<\Sig\hth^t, \hth^t\>\right | >\xi,
}
for some $\xi>0$. With probability at least $1-c_2e^{-c_3p}$, 
\eqn{
\left | \hmu_2(\hth^t) + \hmu_4(\hth^t)\<\zeta_r(\SigH_S)\hth^t, \hth^t\>\right | 
\geq &
\left | \mu_2(\hth^t) + \mu_4(\hth^t)\<\Sig\hth^t, \hth^t\>\right | -
\Big\{ 
\left | \hmu_2(\hth^t) - \mu_2(\hth^t)  \right|\\
&+ \left |\mu_4(\hth^t)\<\Sig\hth^t, \hth^t\> - \hmu_4(\hth^t)\<\Sig\hth^t, \hth^t\>\right |\\
&+ \left | \hmu_4(\hth^t)\<\Sig\hth^t, \hth^t\>  - \hmu_4(\hth^t)\<\zeta_r(\SigH_S)\hth^t, \hth^t\>\right |
\Big\}.
}

Since, we have
\eqn{
\sqrt{\frac{p}{\min\{n/\log(n) , p/\log(p)|S|\}}} \leq \sqrt{\frac{p}{\min\{n/\log(n) , |S|\}}},
}
by the Lemmas \ref{lem::concentration1}-\ref{lem::concentration1bounded} and Lemmas \ref{lem::concentrationBounded1}-\ref{lem::concentration3}, we write
\eqn{
\left | \hmu_2(\hth^t) + \hmu_4(\hth^t)\<\zeta_r(\SigH_S)\hth^t, \hth^t\>\right | 
\geq &
\left | \mu_2(\hth^t) + \mu_4(\hth^t)\<\Sig\hth^t, \hth^t\>\right | 
- \left( c_1+ \| \hth^t\|_2^2 \| \Sig\|_2\right)\sqrt{\frac{p\log(n)}{n}}\\
&-B_4\|\hth^t\|_2^2 \left\|\zeta_r(\SigH_S)-\Sig \right\|_2,\\
\geq &
\left | \mu_2(\hth^t) + \mu_4(\hth^t)\<\Sig\hth^t, \hth^t\>\right | 
- C\sqrt{\frac{p}{\min\{n/\log(n) , |S|\}}},\\
\geq &
\xi- C\sqrt{\frac{p}{\min\{n/\log(n) , |S|\}}},
}
where $C = \max\{cB_4R^2 , c_1+ R^2 \| \Sig\|_2 \}$.

Therefore, for some constants $c_1,c_2$, with high probability, we have
\eqn{
\left\|\Q^t\right\|_2 \leq &
\frac{1}{\hmu_2(\hth^t)}\left[\left\| \zeta_r(\SigH_S)^{-1}\right\|_2 
+ \frac{|\hmu_4(\hth^t)|\big\|\hth^t\big\|^2_2 }{
\left | \hmu_2(\hth^t) + \hmu_4(\hth^t)\<\zeta_r(\SigH_S)\hth^t, \hth^t\>\right |
}\right],\\
\leq &
\frac{1}{c_R-c_1 \sqrt{\frac{p \log(n)}{n}}}\left[\frac{1}{\hat{\sigma}^2}
+ \frac{B_4 R^2 }{
\xi- C\sqrt{\frac{p}{\min\{n/\log(n) ,|S|\}}}}\right],\\
\leq &
\frac{1}{c_R-c_1 \sqrt{\frac{p \log(n)}{n}}}\left[\frac{1}{\sigma^2-c_2\sqrt{\frac{\log(p)}{|S|}}}
+ \frac{B_4 R^2 }{
\xi- C\sqrt{\frac{p}{\min\{n/\log(n) , |S|\}}}}\right],\\
}
For $n$ and $|S|$ sufficiently large so that we have the following inequalities,
\eqn{
& c_2\sqrt{\frac{\log(p)}{|S|}} \leq \frac{\sigma^2}{2} ,\\
&c_1 \sqrt{\frac{p \log(n)}{n}} \leq \frac{c_R}{2},\\
&C\sqrt{\frac{p}{\min\{n/\log(n) , |S|\}}} \leq \frac{\xi}{2},
}
we obtain
\eqn{
\left\|\Q^t\right\|_2 \leq &
\frac{2}{c_R}\left[\frac{2}{\sigma^2}
+ \frac{2B_4 R^2 }{
\xi}\right]\coloneqq
  \kappa.
}

Finally, we take into account the conditioning on the event $\mathcal{E}$ and conclude the proof.


\begin{proof}[Proof of Corollaries \ref{cor::compositeBounded} and \ref{cor::compositeSG}]
In the following, we provide the proof for Corollary \ref{cor::compositeBounded}. The proof for Corollary \ref{cor::compositeSG} follows from the exact same steps.

The statement of the Theorem \ref{thm::mainBounded} holds on the probability space with a probability lower bounded by $\P(\mathcal{E}) - c/p^2$ for some constant $c$. Let $\QQ$ denote this set, on which the statement of the lemma holds. Note that $\QQ \subset \mathcal{E}$. We have 
\eqn{
\P(\QQ) \geq \P(\mathcal{E}) - c'/p^2.
}
This suggests that the difference between $\QQ$ and $\mathcal{E}$ is small.
By taking expectations on both sides over the set $\QQ$, we obtain,
\eqn{
\E \left [ \| \hth^{t+1} - \th_*\|_2  ;\QQ\right]
\leq 
\kappa \left\{
\D(x,z)
 +c_1 \sqrt{\frac{p}{\min\left\{ p/\log(p)|S|,n/\log(n)\right\}}}\right \}\E \left[\|\hth^t - \th_*\|_2 \right]\\
+\kappa c_2 \E \left[\|\hth^t - \th_*\|_2^2 \right]
  \\
}
where we used \eqn{
\E \left [  \| \hth^{t} - \th_*\|_2^l ; \QQ \right] \leq \E \left [  \| \hth^{t} - \th_*\|_2^l  \right], \ \ \ \ \ l=1,2.
}
Similarly for the iterate $\hth^{t+1}$, we write
\eqn{
\E \left [  \| \hth^{t+1} - \th_*\|_2  \right] =& \E \left [  \| \hth^{t+1} - \th_*\|_2  ; \QQ\right]
+\E \left [  \| \hth^{t+1} - \th_*\|_2  ; \QQ^C\right],\\
\leq& \E \left [  \| \hth^{t+1} - \th_*\|_2  ; \QQ\right] + 2R \P(\QQ^C),\\
\leq& \E \left [  \| \hth^{t+1} - \th_*\|_2  ; \QQ\right] + 2R\left( \P(\mathcal{E}^C) + \frac{c}{p^2}\right),\\
\leq& \E \left [  \| \hth^{t+1} - \th_*\|_2  ; \QQ\right] + \frac{\e}{10},\\
\leq& \E \left [  \| \hth^{t+1} - \th_*\|_2  ; \QQ\right] + \frac{\E \left [  \| \hth^{t} - \th_*\|_2\right]}{10}.\\
}

Combining these two inequalities, we obtain
\eqn{
\E \left [  \| \hth^{t+1} - \th_*\|_2  \right] \leq
\left\{
0.1 +
\kappa
\D(x,z)
 +c_1\kappa \sqrt{\frac{p}{\min\left\{ p/\log(p)|S|,n/\log(n)\right\}}}
   \right \}&\E \left[\|\hth^t - \th_*\|_2 \right] \\
&+c_2\kappa  \E \left[\|\hth^t - \th_*\|_2^2 \right]
.
}
Hence the proof follows.
\end{proof}

\begin{proof}[Proof of Theorem \ref{thm::globalNumIter}]
For a sequence satisfying the following inequality,
\eqn{
\|\hth^{t+1} - \th_*\|_2 
\leq \left(\tau_1 + \tau_2 \|\hth^t - \th_*\|_2 \right)\|\hth^t - \th_*\|_2 ,
}
we observe that 
\eq{
\tau_1 + \tau_2 \|\hth^0 - \th_*\|_2 < 1
}
is a sufficient condition for convergence to 0. Let $\xi \in (\e,1)$ and $t_\xi$ be the last iteration that 
$ \|\hth^t - \th_*\|_2  > \delta$. Then, for $t > t_\xi$
\eqn{
 \|\hth^{t+1} - \th_*\|_2 
\leq& \left(\tau_1 + \tau_2 \|\hth^t - \th_*\|_2 \right) \|\hth^t - \th_*\|_2  ,\\
\leq& \left(\tau_1 + \tau_2 \xi \right) \|\hth^t - \th_*\|_2  .
}
This convergence behavior describes a linear rate and requires at most 
\eqn{
\frac{\log(\e / \xi)}{\log(\tau_1 + \tau_2 \xi)}
}
iterations to reach a tolerance of $\e$. For $t \leq t_\xi$, we have
\eqn{
 \|\hth^{t+1} - \th_*\|_2  
\leq& \left(\tau_1 + \tau_2  \|\hth^t - \th_*\|_2 \right) \|\hth^t - \th_*\|_2  ,\\
\leq& \left(\tau_1/\xi + \tau_2\right) \|\hth^t - \th_*\|_2^2.
}
This describes a quadratic rate and the number of iterations to reach a tolerance of $\xi$ can be upper bounded by
\eqn{
\log_2\left( \frac{\log\left(\delta\left(\tau_1/\xi + \tau_2\right)\right)}
{\log\left(\left(\tau_1/\xi + \tau_2\right)\right) \|\hth^0 - \th_*\|_2  }\right).
}
Therefore, the overall number of iterations to reach a tolerance of $\e$ is upper bounded by
\eqn{
\mathcal{J}(\xi) = \log_2\left( \frac{\log\left(\delta\left(\tau_1/\xi + \tau_2\right)\right)}
{\log\left(\left(\tau_1/\xi + \tau_2\right)\right) \|\hth^0 - \th_*\|_2 }\right)
+
\frac{\log(\e / \xi)}{\log(\tau_1 + \tau_2 \xi)}
}
which is a function of $\xi$. Therefore, we take the minimum over the feasible set.
\end{proof}


\section{Step size selection}\label{sec::step-size}
We carry our analysis from Eq.~\ref{eq::ineq-govern}. The optimal step size would be
\eqn{
\gamma_* = \argmin_\gamma
 \left\| I - \gamma \Q^t \int_0^1\gradS_\th \l (\th_* + \xi(\hth^{t}-\th_*))d\xi\right\|_2 .
}
Using the mean value theorem, 
\eqn{
\gradS_\th\ell(\tilde{\th}) = \int_0^1\gradS_\th \l (\th_* + \xi(\hth^{t}-\th_*))d\xi,
}
where $\tilde{\th} \in [\th_*,\hth^t]$,
and write the governing term as
\eqn{
\left\| I - \gamma \Q^t\gradS_\th\ell(\tilde{\th})\right\|_2 .
}
The above function is piecewise linear in $\gamma$ and it can be minimized by setting
\eqn{
\gamma_* = \frac{2}{\lambda_1\left(\Q^t\gradS_\th\ell(\tilde{\th})\right) + \lambda_p\left(\Q^t\gradS_\th\ell(\tilde{\th})\right)}.
}

Since we don't have access to the optimal value $\th_*$ nor $\tilde{\th}$, we will assume that $\tilde{\th}$ and the current iterate $\hth^t$ are close.

In the regime $n\gg p$, and by our construction of the scaling matrix $\Q^t$, we have
\eqn{
\Q^t \approx \left[ \E[xx^T\ddphi(\<x,\hth^t\>)]\right]^{-1} \ \ \text{and} \ \ \ 
\gradS_\th\ell({\hth^t}) \approx \left[ \E[xx^T\ddphi(\<x,\hth^t\>)]\right].
}

The crucial observation is that $\zeta_r$ function sets the smallest eigenvalue to $\hsig^2$ which overestimates $\sigma^2$ in general. Even though the largest eigenvalue of $\Q^t\gradS_\th\ell(\tilde{\th})$ will be close to 1, the smallest value will be $\sigma^2/\hsig^2$. This will make the optimal step size larger than 1.
Hence, we suggest
\eqn{
\gamma = \frac{2}{1+\sigma^2/\hsig^2}.
}

We also have, by the Weyl's inequality,
\eqn{
\left|\hsig^2 - \sigma^2 \right| \leq \left\| \SigH - \Sig \right\|_2 \leq C\sqrt{\frac{p}{|S|}},
}
with high probability. Whenever $r$ is less than $p/2$, we suggest

\eqn{
\gamma= \frac{2}{1+\frac{\hsig^2 - \O(\sqrt{p/|S|})}{\hsig^2}}.
}

\section{Details of experiments}\label{sec::expDetail}
Table \ref{tab::details} provides the details of the experiments given in Section \ref{sec::experiments}.
\begin{table}[H]
\centering
\footnotesize
\begin{tabular}{|l|l|l|l|r|}
\multicolumn{5}{c}{S3}\\
\hline
    & \multicolumn{2}{c}{LR} \vline  & \multicolumn{2}{c}{LS}\vline\\
\hline
Method    & Elapsed(sec) & Iter		& Elapsed(sec) & Iter\\
\hline
\ALG        &10.637      &  2& 	8.763 	&   4\\
BFGS      &22.885      &  8& 	13.149 & 6 \\
LBFGS      &46.763    & 19 & 19.952 	& 11		 \\
Newton   &55.328     &  2&  	38.150 	& 1		 \\
GD          &865.119   & 493	& 155.155 & 100 	 \\
AGD          &169.473   & 82	& 65.396 & 42 	\\
\hline
\end{tabular}
%
%
\begin{tabular}{|l|l|l|l|r|}
\multicolumn{5}{c}{S20}\\
\hline
    & \multicolumn{2}{c}{LR} \vline  & \multicolumn{2}{c}{LS}\vline\\
\hline
Method    & Elapsed(sec) & Iter		& Elapsed(sec) & Iter\\
\hline
\ALG        &23.158      &  4& 	16.475 	& 10\\
BFGS      &40.258      & 17& 	54.294 & 37 \\
LBFGS      &51.888      & 26 & 33.107 	& 20		 \\
Newton   &47.955     &  2&  	39.328 	& 1		 \\
GD          &1204.015   & 245	& 145.987 & 100 	 \\
AGD          &182.031   & 83	& 56.257 & 38 	\\
\hline
\end{tabular}
%
%
\begin{tabular}{|l|l|l|l|r|}
\multicolumn{5}{c}{CT Slices}\\
\hline
    & \multicolumn{2}{c}{LR} \vline  & \multicolumn{2}{c}{LS}\vline\\
\hline
Method    & Elapsed(sec) & Iter		& Elapsed(sec) & Iter\\
\hline
\ALG        &4.191      &  32& 	1.799 	& 11\\
BFGS      &4.638      & 35& 	4.525 & 37 \\
LBFGS      &26.838      & 217 & 22.679 	& 180		 \\
Newton   &5.730     &  3&  	1.937 	& 1		 \\
GD          &96.142   & 1156	& 61.526 & 721 	 \\
AGD          &96.142   & 880	& 45.864 & 518 	\\
\hline
\end{tabular}
\begin{tabular}{|l|l|l|l|r|}
\multicolumn{5}{c}{Covertype}\\
\hline
    & \multicolumn{2}{c}{LR} \vline  & \multicolumn{2}{c}{LS}\vline\\
\hline
Method    & Elapsed(sec) & Iter		& Elapsed(sec) & Iter\\
\hline
\ALG        &16.113      &  31& 	2.080 	& 5\\
BFGS      &21.916      & 48& 	2.238 & 3 \\
LBFGS      &30.765      & 69 & 2.321 	& 3		 \\
Newton   &122.158     &  40&  	2.164 	&  1		 \\
GD          &194.473   & 446	& 22.738 & 60 	 \\
AGD          &80.874   & 186	& 32.563 & 77 	\\
\hline
\end{tabular}
\caption{\label{tab::details}
Details of the experiments presented in Figures \ref{fig::plotSyn} and \ref{fig::plotReal}.}
\end{table}


\section{Useful lemmas}
\begin{lemma}[Gautschi's Inequality]\label{lem::gautschi}
Let $\Gamma$ denote the \emph{Gamma} function. Then, for $r \in (0,1)$, we have
\[
z^{1-r} < \frac{\Gamma(z+1)}{\Gamma(z+r)} < (1+z)^{1-r}.
\]
\end{lemma}

\begin{lemma}\label{lem::darthvader}
Let $Z$ be a random variable with a density function $f$ and cumulative distribution function $F$. If $F^C=1-F$, then,
\eqn{
\left|\E [Z\ind_{\{ |Z|>t\}}]\right|\leq t \P(|Z|>t) + \int^\infty_t\P(|Z|>z)dz.
}
\end{lemma}
\begin{proof}
We write,
\eqn{
\E [Z\ind_{\{ |Z|>t\}}] = \int_t^\infty zf(z)dz + \int^{-t}_{-\infty}zf(z)dz.
}
Using integration by parts, we obtain
\eqn{
\int zf(z)dz =& -zF^C (z) + \int F^C(z)dz,\\
=& zF (z) - \int F(z)dz.
}
Since $\lim_{z\to\infty}zF^C(z)=\lim_{z\to-\infty}zF(z)=0$, we have
\eqn{
 \int_t^\infty zf(z)dz =& t F^C(t) + \int^\infty_tF^C(z)dz,\\
  \int_{-\infty}^{-t} zf(z)dz =& -t F(-t) - \int_{-\infty}^{-t}F(z)dz,\\
  =& -t F(-t) - \int_{t}^{\infty}F(-z)dz.
}
Hence, we obtain the following bound,
\eqn{
\left|\E [Z\ind_{\{ |Z|>t\}}]\right| =& \left|t F^C(t) + \int^\infty_tF^C(z)dz -t F(-t) - \int_{t}^{\infty}F(-z)dz\right|,\\
\leq& t\left( F^C(t) +  F(-t)\right) + \left(\int^\infty_tF^C(z)+F(-z)dz\right),\\
\leq& t \P(|Z|>t) + \int^\infty_t\P(|Z|>z)dz.
}
\end{proof}

\begin{lemma}\label{lem::integralBound}
For positive constants $c_1,c_2$, we have
\eqn{
\int_{c_1}^\infty e^{-c_2 t^{2/3}}dt \leq
\left \{ \frac{3c_1^{1/3}}{2c_2}
  + \frac{3}{4c^2_2 c_1^{1/3}}
\right \}e^{-c_2 c_1^{2/3}}
}
\end{lemma}
\begin{proof}
By the change of variables $t^{2/3}=x^2$, we get
\eqn{
\int_{c_1}^\infty e^{-c_2 t^{2/3}}dt =3\int_{c^{1/3}_1}^\infty x^{2}e^{-c_2 }x^2dx.
}
Next, we notice that 
\eqn{
de^{-c_2x^2} = -2c_2xe^{-c_2x^2}dx.
}
Hence, using the integration by parts, we have
\eqn{
\int_{c_1}^\infty e^{-c_2 t^{2/3}}dt
=
\frac{3}{2c_2}\left \{ 
c_1^{1/3} e^{-c_2 c_1^{2/3}} + \int_{c_1^{1/3}}^\infty e^{-c_2 x^2}dx
\right \}.
}
We will find an upper bound on the second term. Using the change of variables,  $x=y+c_1^{1/3}$, we obtain
\eqn{
\int_{c_1^{1/3}}^\infty e^{-c_2 x^2}dx =& \int_0^\infty e^{-c_2 \left(y+c_1^{1/3}\right)^2}dy,\\
\leq& e^{-c_2 c_1^{2/3}}\int_0^\infty e^{-2c_2 yc_1^{1/3}}dy,\\
=& \frac{e^{-c_2 c_1^{2/3}}}{2c_2 c_1^{1/3}}.
}
Combining the above results, we complete the proof.
\end{proof}

\begin{lemma}[\cite{vershynin2010introduction}]\label{lem::epsNet}
Let $X$ be a symmetric $p\times p$ matrix, and let $T_\epsilon$ be an $\epsilon$-net over $S^{p-1}$. Then,
\eqn{
\| X\|_2 \leq\frac{1}{1-2\epsilon}\sup_{v\in T_\epsilon}\left|\<Xv,v\>\right|.
}
\end{lemma}

\begin{lemma}\label{lem::sphere}
Let $B_p(R)\subset \reals^p$ be the ball of radius $R$ centered at the origin and $T_\epsilon$ be an $\epsilon$-net over $B_p(R)$. Then,
\eqn{
|T_\epsilon|\leq \left(\frac{R\sqrt{p}}{\epsilon}\right)^p.
}
\end{lemma}
\begin{proof}[Proof of Lemma \ref{lem::sphere}]
A similar proof appears in \cite{van2000asymptotic}. The set $B_p(R)$ can be contained in a $p$-dimensional cube of size $2R$. Consider a grid over this cube with mesh width $2\epsilon/\sqrt{p}$. Then $B_p(R)$ can be covered with at most $(2R/(2\epsilon/\sqrt{p}))^p$ many cubes of edge length $2\epsilon/\sqrt{p}$. If ones takes the projection of the centers of such cubes onto $B_p(R)$ and considers the circumscribed balls of radius $\epsilon$, we may conclude that $B_p(R)$ can be covered with at most
$$\left(\frac{2R}{2\epsilon/\sqrt{p}}\right)^p$$
many balls of radius $\epsilon$.
\end{proof}

\begin{lemma}\label{lem::epsilon}
For $a,b>0$, and $\epsilon$ satisfying 
\eqn{
\e = \left\{\frac{a}{2}\log\left(\frac{2b^2}{a}\right)\right\}^{1/2}\ \ \ \ \text{and }\ \ \ \ \frac{2}{a}b^2 > e,
}
we have 
$\e^2 \geq a\log(b/\e)$.
\end{lemma}

\begin{proof}
Since $a,b>0$ and $x\to e^x$ is a monotone increasing function,
the above inequality condition is equivalent to
\eqn{
\frac{2\e^2}{a} e^{\frac{2\e^2}{a}} \geq \frac{2b^2}{a}.
}
Now, we define the function $f(w)= we^w$ for $w>0$. $f$ is continuous and invertible on $[0,\infty)$. 
Note that $f^{-1}$ is also a continuous and increasing function for $w>0$. Therefore, we have
\eqn{
\e^2 \geq \frac{a}{2}f^{-1}\left(\frac{2b^2}{a}\right)
}
Observe that the smallest possible value for $\e$ would be simply the square root of ${a}f^{-1}\left({2b^2}/{a}\right)/{2}$. For simplicity, we will obtain a more interpretable expression for $\e$. By the definition of $f^{-1}$, we have
\eqn{
\log(f^{-1}(y)) + f^{-1}(y) = \log (y).
}
Since the condition on $a$ and $b$ enforces $f^{-1}(y)$ to be larger than 1, we obtain the simple inequality that 
\eqn{
f^{-1}(y)  \leq \log (y).
}
Using the above inequality, if $\e$ satisfies 
\eqn{
\e^2 = \frac{a}{2}\log\left(\frac{2b^2}{a}\right)\geq \frac{a}{2}g\left(\frac{2b^2}{a}\right),
}
we obtain the desired inequality.
\end{proof}

\end{document}